\documentclass[twoside]{article} \usepackage{aistats2017}

\usepackage{hyperref}       
\usepackage{url}            
\usepackage{booktabs}       
\usepackage{amsfonts}       
\usepackage{nicefrac}       
\usepackage{microtype}      
\usepackage{amsmath}
\usepackage{subcaption}
\usepackage{color}
\usepackage{floatrow}

\usepackage{mathtools}
\usepackage{graphicx}
\usepackage{comment}
\usepackage{amssymb}
\usepackage{collectbox}
\usepackage{framed}
\usepackage{upgreek}
\usepackage{bm}

\newtheorem{theorem}{Theorem}
\newtheorem{lemma}{Lemma}
\newtheorem{definition}{Definition}
\newtheorem{remark}{Remark}

\newenvironment{proof}{\paragraph{Proof:}}{\hfill$\square$}
\newcommand{\todo}[2]{\textcolor{red}{TODO(#1): #2}}

\begin{document}

%

%

\twocolumn[
\aistatstitle{Structured adaptive and random spinners for fast machine learning computations}

\vspace{-0.2in}
\aistatsauthor{ Mariusz Bojarski\footnotemark[1] \And Anna Choromanska$^1$ \And Krzysztof Choromanski$^1$}

\aistatsaddress{ NVIDIA\\ mbojarski@nvidia.com \And Courant Inst. of Math. Sciences, NYU\\ achoroma@cims.nyu.edu \And Google Brain Robotics\\ kchoro@google.com } 
\vspace{-0.25in}
\aistatsauthor{Francois Fagan$^1$ \And Cedric Gouy-Pailler$^1$ \And Anne Morvan$^1$ }

\aistatsaddress{ Columbia University\\ ff2316@columbia.edu \And CEA, LIST, LADIS\\ cedric.gouy-pailler@cea.fr \And CEA, LIST, LADIS \\ and Universite Paris-Dauphine\\ anne.morvan@cea.fr } 
\vspace{-0.25in}
\aistatsauthor{Nourhan Sakr$^{1,}$\footnotemark[2] \And Tamas Sarlos$^1$ \And Jamal Atif 
}

\aistatsaddress{ Columbia University\\ n.sakr@columbia.edu \And Google Research\\ stamas@google.com \And Universite Paris-Dauphine\\ jamal.atif@dauphine.fr } 
]

\footnotetext[1]{equal contribution}
\footnotetext[2]{partly supported by NSF grant CCF-1421161}

\begin{abstract}
We consider an efficient computational framework for speeding up several machine learning algorithms with almost no loss of accuracy.  The 
proposed framework relies on projections via structured matrices that we call \textit{Structured Spinners}, which are formed as products of three structured matrix-blocks that incorporate rotations. The approach is highly generic, i.e. i) structured matrices under consideration can either be fully-randomized or learned, ii) our structured family contains as special cases all previously considered structured schemes, iii) the setting extends to the non-linear case where the projections are followed by non-linear functions, and iv) the method finds numerous applications including kernel approximations via random feature maps, dimensionality reduction algorithms, new fast cross-polytope LSH techniques, deep learning, convex optimization algorithms via Newton sketches, quantization with random projection trees, and more. The proposed framework comes with theoretical guarantees characterizing the capacity of the structured model in reference to its unstructured counterpart and is based on a general theoretical principle that we describe in the paper. As a consequence of our theoretical analysis, we provide the first theoretical guarantees for one of the most efficient existing LSH algorithms based on the $\textbf{HD}_{3}\textbf{HD}_{2}\textbf{HD}_{1}$ structured matrix \cite{indyk2015}. The exhaustive experimental evaluation confirms the accuracy and efficiency of structured spinners for a variety of different applications.
\end{abstract}

\section{Introduction}

A striking majority of machine learning frameworks performs projections of input data via matrices of parameters, where the obtained projections are often passed to a possibly highly nonlinear function. In the case of randomized machine learning algorithms, the projection matrix is typically Gaussian with i.i.d. entries taken from $\mathcal{N}(0,1)$. Otherwise, it is learned through the optimization scheme. A plethora of machine learning algorithms admits this form. In the randomized setting, a few examples include variants of the Johnson-Lindenstrauss Transform applying random projections to reduce data dimensionality while approximately preserving Euclidean distance ~\cite{ailon2006approximate, liberty_jlt, ailon2013}, kernel approximation techniques based on random feature maps produced from linear projections with Gaussian matrices followed by nonlinear mappings \cite{rahimi}, \cite{le2013fastfood, chor_sind_2016, Huang_Kernel_DNN_ICASSP2014}, \cite{choromanska2015binary}, LSH-based schemes~\cite{v008a014,charikar, terasawa}, including the fastest known variant of the cross-polytope LSH~\cite{indyk2015}, algorithms solving convex optimization problems with random sketches of Hessian matrices \cite{pilanci, pilanci2}, quantization techniques using random projection trees, where splitting in each node is determined by a projection of data onto Gaussian direction \cite{dasgupta08}, and many more. 

The classical example of machine learning nonlinear models where linear projections are learned is a multi-layered neural network~\cite{0483bd9444a348c8b59d54a190839ec9, Goodfellow-et-al-2016-Book}, where the operations of linear projection via matrices with learned parameters followed by the pointwise nonlinear feature transformation are the building blocks of the network's architecture. These two operations are typically stacked multiple times to form a deep network. 

The computation of projections takes $\Theta(mn|\mathcal{X}|)$ time, where $m\times n$ is the size of the projection matrix, and $|\mathcal{X}|$ denotes the number of data samples from a dataset $\mathcal{X}$. In case of high-dimensional data, this comprises  a significant fraction of the overall computational time, while storing the projection matrix frequently becomes a bottleneck in terms of space complexity. 

In this paper, we propose the remedy for both problems, which relies on replacing the aforementioned algorithms by their ``structured variants''. The projection is performed by applying a structured matrix
from the family that we introduce as \textit{Structured Spinners}. 
Depending on the setting, the structured matrix is either learned or its parameters are taken from a random distribution (either continuous or discrete if further compression is required). Each structured spinner is a product of three matrix-blocks that incoporate rotations.
A notable member of this family is a matrix of the form $\textbf{HD}_{3}\textbf{HD}_{2}\textbf{HD}_{1}$, where $\textbf{D}_{i}$s are either random diagonal $\pm 1$-matrices or adaptive diagonal matrices and $\textbf{H}$ is the Hadamard matrix. This matrix is used in the fastest known cross-polytope LSH method introduced in \cite{indyk2015}.

In the structured case, the computational speedups are significant, i.e. projections can be calculated in $o(mn)$ time, often in $O(n\log m)$ time if Fast Fourier Transform techniques are applied. At the same time, using matrices from the family of structured spinners leads to the reduction of space complexity to sub-quadratic, usually at most linear, or sometimes even constant. 

The key contributions of this paper are: 
\begin{itemize}
\item The family of structured spinners providing a highly parametrized class of structured methods and, as we show in this paper, with applications in various randomized settings such as: kernel approximations via random feature maps, dimensionality reduction algorithms, new fast cross-polytope LSH techniques, deep learning, convex optimization algorithms via Newton sketches, quantization with random projection trees, and more.
\item A comprehensive theoretical explanation of the effectiveness of the structured approach based on structured spinners. Such analysis was provided in the literature before for a strict subclass of a very general family of structured matrices that we consider in this paper, i.e.  the proposed family of structured spinners contains all previously considered structured matrices as special cases, including the recently introduced $P$-model \cite{chor_sind_2016}. To the best of our knowledge, we are the first to theoretically explain the effectiveness of structured neural network architectures. Furthermore, we provide first theoretical guarantees for a wide range of discrete structured transforms, in particular for the fastest known cross-polytope LSH method \cite{indyk2015} based $\textbf{HD}_{3}\textbf{HD}_{2}\textbf{HD}_{1}$ discrete matrices.
\end{itemize}

Our theoretical methods in the random setting apply the relatively new Berry-Esseen type Central Limit Theorem results for random vectors.

Our theoretical findings are supported by empirical evidence regarding the accuracy and efficiency of structured spinners in a wide range of different applications. Not only do structured spinners cover all already existing structured transforms as special instances, but also many other structured matrices that can be applied in all aforementioned applications.


\section{Related work}
\label{sec:related_work}

This paper focuses on structured matrices, which were previously explored in the literature mostly in the context of the Johnson-Lindenstrauss Transform (JLT)~\cite{johnson84extensionslipschitz}, where the high-dimensional data is linearly transformed and embedded into a much lower dimensional space while approximately preserving the Euclidean distance between data points. Several extensions of JLT have been proposed, e.g.~\cite{liberty_jlt, ailon2013, ailon2006approximate,vybiral2011variant}. Most of these structured constructions involve sparse~\cite{ailon2006approximate,Dasgupta:2010:SJL:1806689.1806737} or circulant matrices~\cite{vybiral2011variant,journals/rsa/HinrichsV11} providing computational speedups and space compression.

More recently, the so-called $\Psi$-regular structured matrices (Toeplitz and circulant matrices belong to this wider family of matrices) were used to approximate angular distances~\cite{choromanska2015binary} and signed Circulant Random Matrices were used to approximate Gaussian kernels~\cite{feng}. Another work~\cite{chor_sind_2016} applies structured matrices coming from the so-called \textit{P-model}, which further generalizes the $\Psi$-regular family, to speed up random feature map computations of some special kernels (angular, arc-cosine and Gaussian). These techniques did not work for discrete structured constructions, such as the $\textbf{HD}_{3}\textbf{HD}_{2}\textbf{HD}_{1}$ matrices, or their
direct non-discrete modifications, since they require matrices with low (polylog) chromatic number of the corresponding coherence graphs.

Linear projections are used in the LSH setting to construct codes for given datapoints which  speed up such tasks as approximate nearest neighbor search. A notable set of methods are the so-called cross-polytope techniques introduced in \cite{terasawa} and their aforementioned discrete structured variants proposed in \cite{indyk2015} that are based on the Walsh-Hadamard transform. Before our work, they were only experimentally verified to produce good quality codes.





Furthermore, a recently proposed technique based on the so-called \textit{Newton Sketch} provides yet another example of application for structured matrices. The method~\cite{pilanci, pilanci2} is used for speeding up algorithms solving convex optimization problems by approximating Hessian matrices using so-called \textit{sketch matrices}. Initially, the sub-Gaussian sketches based on i.i.d. sub-Gaussian random variables
were used. The disadvantage of the sub-Gaussian sketches lies in the fact that computing the sketch of the given matrix 
of size $n \times d$ requires $O(mnd)$ time, where $m \times n$ in the size of the sketch matrix. Thus the method is too slow in practice and could be accelerated with the use of  structured matrices. Some structured approaches were already considered, e.g. sketches based on randomized orthonormal 
systems were proposed~\cite{pilanci}. 

All previously considered methods focus on the randomized setting, whereas the structured matrix instead of being learned is fully random. In the context of adaptive setting, where the parameters are being learned instead, we focus in this paper on multi-layer neural networks. We emphasize though that our approach is much more general and extends beyond this setting. 
Structured neural networks were considered before, for instance in ~\cite{Yang2015}, where the so-called \textit{Deep Fried Neural Convnets} were proposed.
Those architectures are based on the adaptive version of the
Fastfood transform used for approximating various kernels~\cite{le2013fastfood}, which is a special case of structured spinner matrices.

Deep Fried Convnets apply adaptive structured matrices
for fully connected layers of the convolutional networks. The structured matrix is of the form: $\textbf{SHG} \bm{\Pi} \textbf{HB}$, where $\textbf{S}$, $\textbf{G}$, and $\textbf{B}$ are adaptive diagonal matrices, $\bm{\Pi}$ is a random permutation matrix, and $\textbf{H}$ is the Walsh-Hadamard matrix. The method reduces the storage and computational costs of matrix multiplication step from, often prohibitive, $\mathcal{O}(nd)$ down to $\mathcal{O}(n)$ storage and $\mathcal{O}(n\log d)$ computational cost, where $d$ and $n$ denote the size of consequitive layers of the network. At the same time, this approach does not sacrifice the network's predictive performance. Another work~\cite{Moczulski2015} that offers an improvement over Deep Fried Convnets, looks at a structured matrix family that is very similar to $\textbf{HD}_{3}\textbf{HD}_{2}\textbf{HD}_{1}$ (however is significantly less general than the family of structured spinners). Their theoretical results rely on the analysis in~\cite{Huhtanen2015}.

The Adaptive Fastfood approach elegantly complements previous works dedicated to address the problem of huge overparametrization of deep models with structured matrices, e.g. the method of~\cite{NIPS2013_5025} represents the parameter matrix as a product of two low rank factors and, similarly to Adaptive Fastfood, applies both at train and test time,~\cite{DBLP:conf/icassp/SainathKSAR13} introduces low-rank matrix factorization to reduce the size of the fully connected layers at train time, and~\cite{Li13} uses low-rank factorizations with SVD after training the full model. These methods, as well as approaches that consider kernel methods in deep learning~\cite{NIPS2009_3628,Mairal:2014:CKN:2969033.2969120,dai2014scalable,Huang_Kernel_DNN_ICASSP2014}, are conveniently discussed in~\cite{Yang2015}.

Structured neural networks are also considered in \cite{sindhwani_15}, where low-displacement rank matrices are applied for linear projections.
The advantage of this approach over Deep Fried Convnets is due to the high parametrization of the family of low-displacement rank matrices allowing the adjustment of the number of parameters learned based on accuracy and speedup requirements.

The class of structured spinners proposed in this work is more general than Deep Fried Convnets or low displacement rank matrices, but it also provides much easier structured constructions, such as $\textbf{HD}_{3}\textbf{HD}_{2}\textbf{HD}_{1}$ matrices, where $\textbf{D}_{i}$s are adaptive diagonal matrices. Furthermore, to the best of our knowledge we are the first to prove theoretically that structured neural networks learn good quality models, by analyzing the capacity of the family of structured spinners.

\section{The family of \textit{Structured Spinners}}
\label{sec:model}

Before introducing the family of structured spinners, we explain notation.
If not specified otherwise, matrix $\textbf{D}$
is a random diagonal matrix with diagonal entries taken independently at random from $\{-1,+1\}$.
By $\textbf{D}_{t_{1},...,t_{n}}$ we denote the diagonal matrix with diagonal equal to $(t_{1},...,t_{n})$.
For a matrix $\textbf{A}=\{a_{i,j}\}_{i,j=1,...,n} \in \mathbb{R}^{n \times n}$, we denote by $\|\textbf{A}\|_{F}$ its Frobenius norm, i.e. $\|\textbf{A}\|_{F} = \sqrt{\sum_{i,j \in \{1,...,n\}}a_{i,j}^{2}}$, and by $\|\textbf{A}\|_{2}$ its spectral norm, i.e. $\|\textbf{A}\|_{2} = \sup_{\textbf{x} \neq 0} \frac{\|\textbf{Ax}\|_{2}}{\|\textbf{x}\|_{2}}$.
We denote by $\textbf{H}$ the $L_{2}$-normalized Hadamard matrix. We say that $\textbf{r}$ is a random Rademacher vector if every element of $\textbf{r}$ is chosen independently at random from $\{-1,+1\}$.

For a vector $\textbf{r} \in \mathbb{R}^{k}$ and $n>0$ let $\textbf{C}(\textbf{r},n) \in \mathbb{R}^{n \times nk}$ be a matrix,
where the first row is of the form $(\textbf{r}^{T},0,...,0)$ and each subsequent row is obtained from the previous one by right-shifting in a circulant manner the previous one by $k$. For a sequence of matrices $\textbf{W}^{1},...,\textbf{W}^{n} \in \mathbb{R}^{k \times n}$ we denote by $\textbf{V}(\textbf{W}^{1},...,\textbf{W}^{n}) \in \mathbb{R}^{nk \times n}$ a matrix obtained by vertically stacking matrices: $\textbf{W}^{1},...,\textbf{W}^{n}$.


Each structured matrix $\textbf{G}_{struct} \in \mathbb{R}^{n \times n}$ from the family of structured spinners is a product of three main structured components/blocks, i.e.:
\begin{equation}
\textbf{G}_{struct} = \textbf{M}_{3}\textbf{M}_{2}\textbf{M}_{1},
\end{equation}
where matrices $\textbf{M}_{1}, \textbf{M}_{2}$ and $\textbf{M}_{3}$ satisfy conditions:

\begin{figure*}[htp!]
  \center
\includegraphics[width = 2.6in]{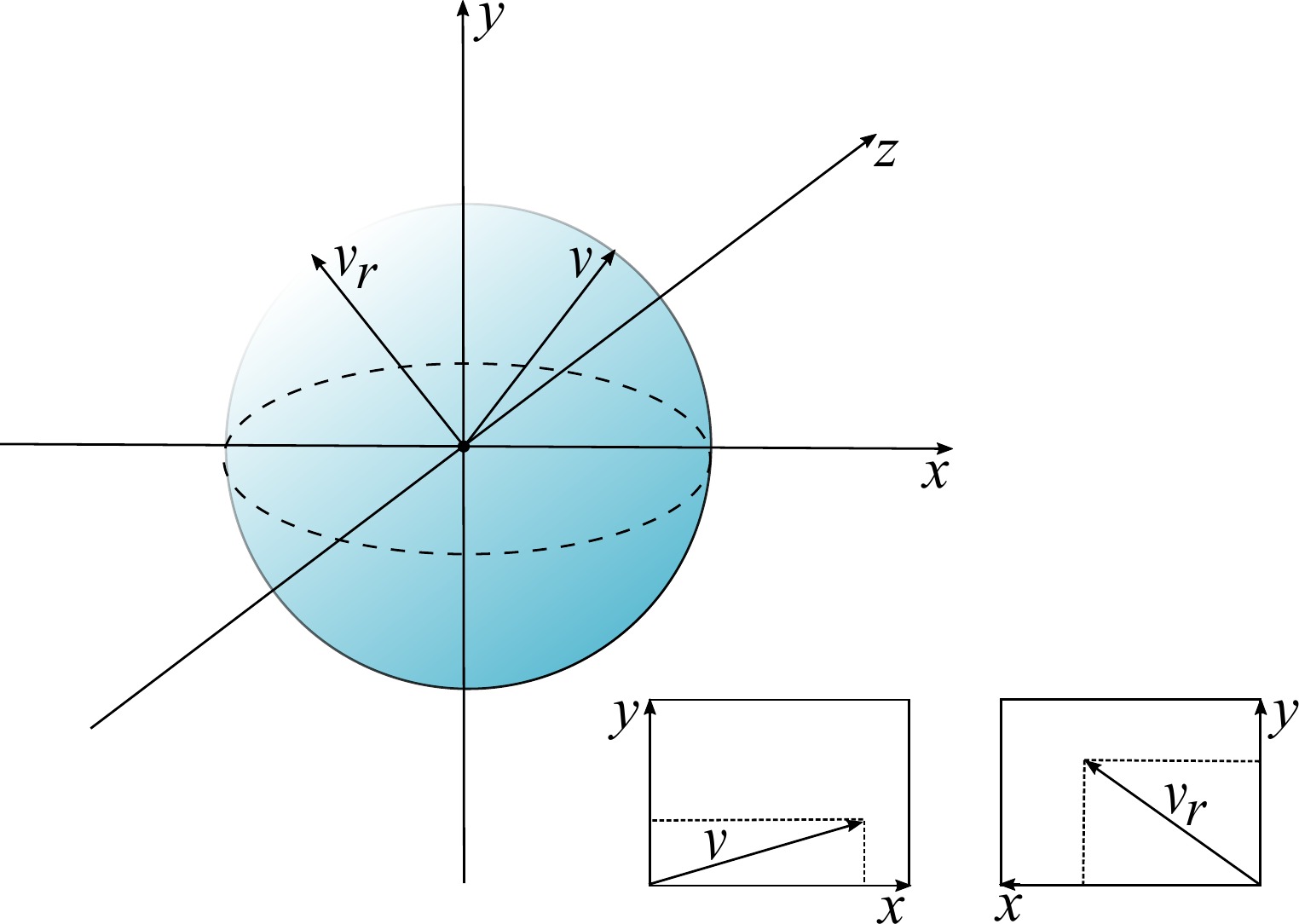}
\hspace{-0.05in}\includegraphics[width = 2.1in]{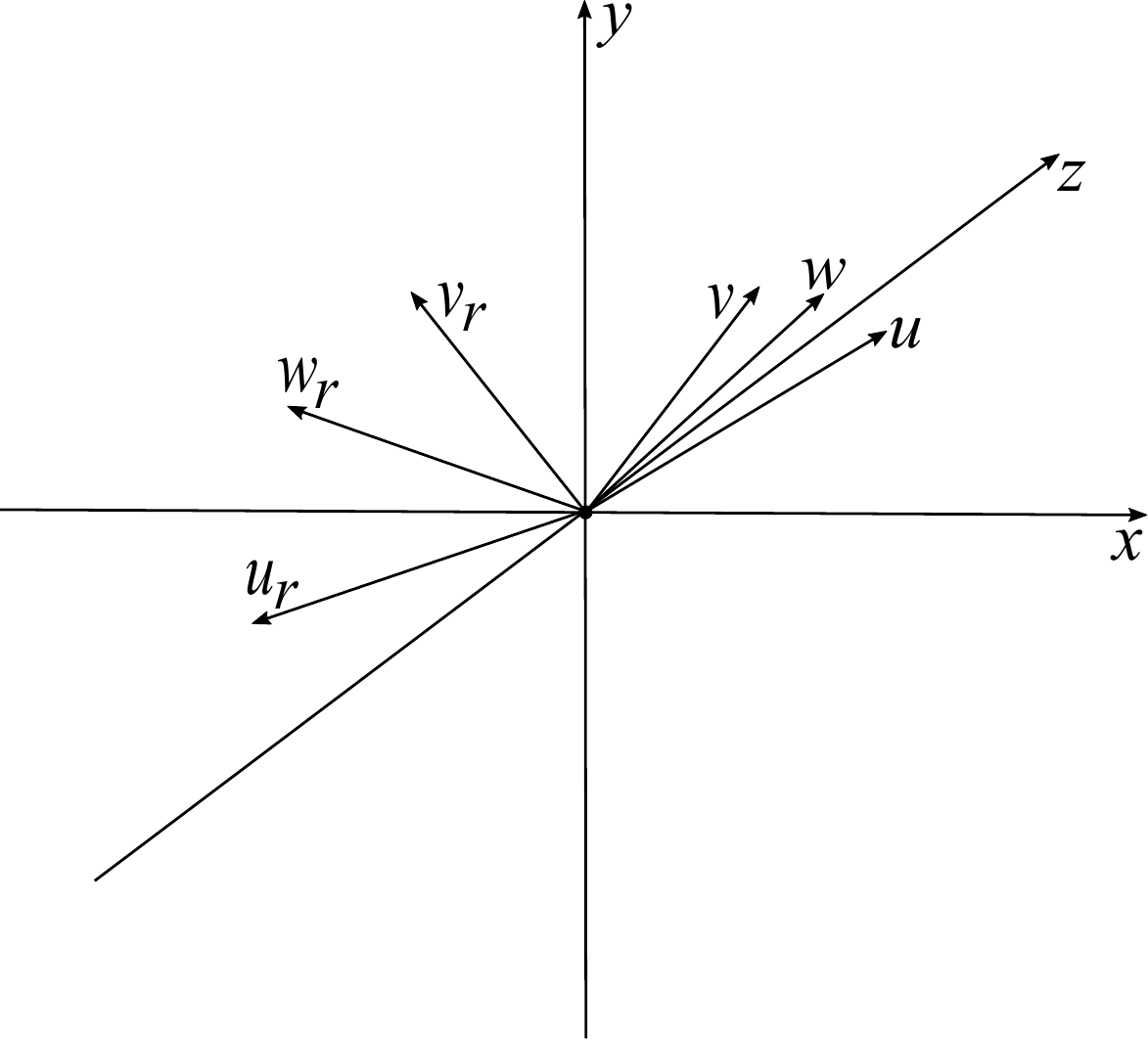} 
\includegraphics[width = 2in]{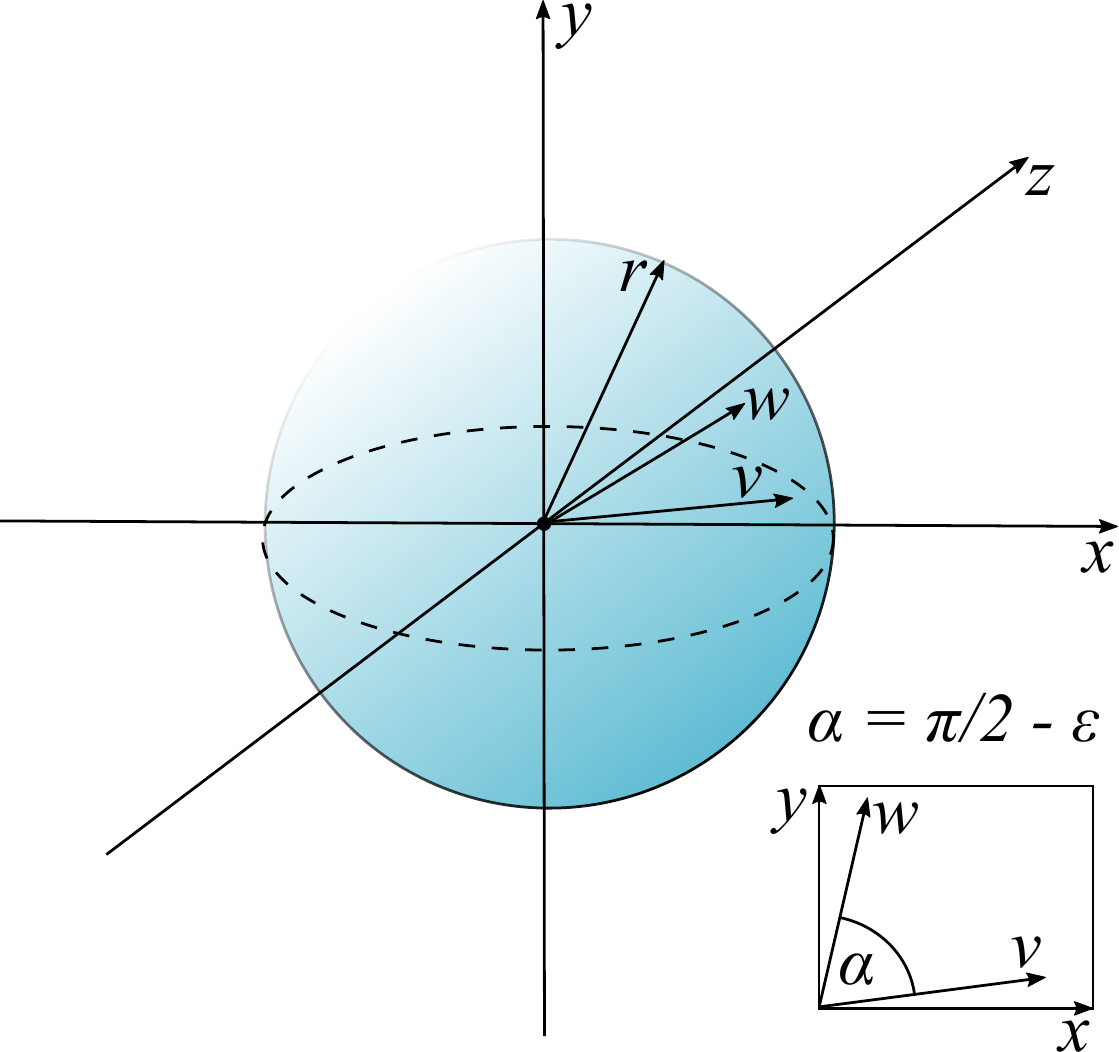} 
\vspace{-0.25in}
\caption{Pictorial explanation of the role of three matrix-blocks in the construction of the structured spinner. Left picture: $\textbf{M}_{1}$ rotates $\textbf{v}$ such that the rotated version $\textbf{v}_{r}$ is balanced. Middle picture: $\textbf{M}_{2}$ transforms vectors $\textbf{v},\textbf{w},\textbf{u}$ such that their images $\textbf{v}_{r},\textbf{w}_{r},\textbf{u}_{r}$ are near-orthogonal. Right picture: The projections of the random vector $\textbf{r}$ onto such two near-orthogonal vectors $\textbf{v}$, \textbf{w} are near-independent.\vspace{-0.15in}}
\label{fig:structspin}
\end{figure*}

\begin{framed}
\textbf{Condition 1:} Matrices: $\textbf{M}_{1}$ and $\textbf{M}_{2}\textbf{M}_{1}$ are $(\delta(n),p(n))$-balanced isometries. \\
\textbf{Condition 2:} $\textbf{M}_{2} = \textbf{V}(\textbf{W}^{1},...,\textbf{W}^{n})\textbf{D}_{\rho_{1},...,\rho_{n}}$ for some $(\Delta_{F},\Delta_{2})$-smooth set: ${\textbf{W}^{1},...,\textbf{W}^{n}} \in \mathbb{R}^{k \times n}$ and some i.i.d sub-Gaussian random variables $\rho_{1},...,\rho_{n}$ with sub-Gaussian norm $K$. \\
\textbf{Condition 3:} $\textbf{M}_{3} = \textbf{C}(\textbf{r},n)$
for $\textbf{r} \in \mathbb{R}^{k}$, where $\textbf{r}$ is random Rademacher/Gaussian in the random setting and is learned in the adaptive setting.
\end{framed}

Matrix $\textbf{G}_{struct}$ is a structured spinner with parameters: $\delta(n),p(n),K,\Lambda_{F},\Lambda_{2}$.
We explain the introduced conditions below. 

\begin{definition}[$(\delta(n),p(n))$-balanced matrices]
A randomized matrix $\textbf{M} \in \mathbb{R}^{n \times m}$ is $(\delta(n),p(n))$-balanced if for every $\textbf{x} \in \mathbb{R}^{m}$ with $\|\textbf{x}\|_{2} = 1$ we have: $\mathbb{P}[\|\textbf{M}\textbf{x}\|_{\infty} > \frac{\delta(n)}{\sqrt{n}}] \leq p(n)$.
\end{definition}

\begin{remark}
\label{balanceness_remark}
One can take as $\textbf{M}_{1}$ a matrix $\textbf{HD}_{1}$
since, as we will show in the Supplement, matrix $\textbf{HD}_{1}$ is $(\log(n),2ne^{-\frac{\log^{2}(n)}{8}})$-balanced. 
\end{remark}

\begin{definition}[$(\Delta_{F},\Delta_{2})$-smooth sets]
A deterministic set of matrices $\textbf{W}^{1},...,\textbf{W}^{n} \in \mathbb{R}^{k \times n}$ is $(\Lambda_{F},\Lambda_{2})$-smooth if:
\begin{itemize}
\vspace{-0.1in}
\item $\|\textbf{W}^{i}_{1}\|_{2} = .. = \|\textbf{W}^{i}_{n}\|_{2}$
      for $i = 1,...,n$, where $\textbf{W}^{i}_{j}$ stands for the  
      $j^{th}$ column of $\textbf{W}^{i}$,
      \vspace{-0.1in}
 \item for $i \neq j$ and $l = 1,...,n$ we have: 
       $(\textbf{W}^{i}_{l})^{T} \cdot \textbf{W}^{j}_{l} = 0$, 
       \vspace{-0.2in}
\item $\max_{i,j} \|(\textbf{W}^{j})^{T}\textbf{W}^{i}\|_{F} \leq \Lambda_{F}$ and \\ $\max_{i,j} \|(\textbf{W}^{j})^{T}\textbf{W}^{i}\|_{2} \leq \Lambda_{2}.$
\end{itemize}
\end{definition}

\begin{remark}
If the unstructured matrix $\textbf{G}$ has rows taken from the general multivariate Gaussian distribution with diagonal covariance matrix $\Sigma \neq \textbf{I}$ then one needs to rescale vectors $\textbf{r}$ accordingly.
For clarity, we assume here that $\Sigma = \textbf{I}$ and we present our theoretical results for that setting.
\end{remark}

All structured matrices previously considered are special cases of a wider family of structured spinners (for clarity, we will explicitly show it for some important special cases). We have:

\begin{lemma}
\label{simple_lemma}
The following matrices: $\textbf{G}_{circ}\textbf{D}_{2}\textbf{HD}_{1}$, $\sqrt{n}\textbf{HD}_{3}\textbf{HD}_{2}\textbf{HD}_{1}$
and $\sqrt{n}\textbf{HD}_{g_{1},...,g_{n}}\textbf{HD}_{2}\textbf{HD}_{1}$, where $\textbf{G}_{circ}$ is Gaussian circulant, are valid structured spinners for $\delta(n) = \log(n)$, $p(n) = 2ne^{-\frac{\log^{2}(n)}{8}}$, $K = 1$, $\Lambda_{F} = O(\sqrt{n})$ and $\Lambda_{2} = O(1)$. The same is true if one replaces $\textbf{G}_{circ}$ by a Gaussian Hankel or Toeplitz matrix.
\end{lemma}

\subsection{The role of three blocks $\bf M_1$, $\bf M_2$, and $\bf M_3$}

The role of blocks $\textbf{M}_{1}$, $\textbf{M}_{2}$, $\textbf{M}_{3}$ can be intuitively explained. Matrix $\textbf{M}_{1}$ makes vectors ``balanced'', so that there is no dimension that carries too much of the $L_{2}$-norm of the vector. The balanceness property was already applied in the structured setting \cite{ailon2006approximate}.

The role of $\textbf{M}_{2}$ is more subtle and differs between adaptive and random settings.
In the random setting, the cost of applying the structured mechanism is the loss of independence.
For instance, the dot products of the rows of a circulant Gaussian matrix with a given vector $\textbf{x}$ are no longer independent, as it is the case in the fully random setup. Those dot products can be expressed as a dot product of a fixed Gaussian row with different vectors $\textbf{v}$. Matrix $\textbf{M}_{2}$ makes these vectors close to orthogonal. In the adaptive setup, the ``close to orthogonality'' property is replaced by the independence property.

Finally, matrix $\textbf{M}_{3}$ defines the capacity of the entire structured transform by providing a vector of parameters (either random or to be learned). The near-independence of the aforementioned dot products in the random setting is now implied by the near-orthogonality property achieved by $\textbf{M}_{2}$ and the fact that the projections of the Gaussian vector or the random Rademacher vector onto ``almost orthogonal directions'' are ``close to independent''. The role of the three matrices is described pictorially in Figure~\ref{fig:structspin}.

\subsection{Stacking together \textit{Structured Spinners}}
\label{subsec:block}

We described structured spinners as square matrices, but in practice we are not restricted to those, i.e. one can construct an $m \times n$ structured spinner for $m \leq n$ from the square $n \times n$ structured spinner by taking its first $m$ rows. We can then stack vertically these independently constructed $m \times n$ matrices to obtain an $k \times n$ matrix for both: $k \leq n$ and $k > n$. We think about $m$ as another parameter of the model that tunes the ``structuredness'' level, i.e. larger values of $m$ indicate more structured approach while smaller values lead to more random matrices ($m=1$ case is the fully unstructured one). 

\section{Theoretical results}
\label{sec:theory}

We now show that structured spinners can replace their unstructured counterparts in many machine learning algorithms with minimal loss of accuracy.

Let $\mathcal{A}_{\mathcal{G}}$ be a machine learning algorithm applied to a fixed dataset $\mathcal{X} \subseteq \mathbb{R}^{n}$ and parametrized by a set $\mathcal{G}$ of matrices $\textbf{G} \in \mathcal{R}^{m \times n}$, where each $\textbf{G}$ is either learned or Gaussian with independent entries taken from $\mathcal{N}(0,1)$. Assume furthermore, that $\mathcal{A}_{\mathcal{G}}$ consists of functions $f_{1},...,f_{s}$, where each $f_{i}$ applies a certain matrix $\textbf{G}_{i}$ from $\mathcal{G}$ to vectors from some linear space $\mathcal{L}_{i}$ of dimensionality at most $d$. Note that for a fixed dataset $\mathcal{X}$ function $f_{i}$ is a function of a random vector 
$$\textbf{q}_{f_{i}} = ((\textbf{G}_{i}\textbf{x}^{1})^{T},...,(\textbf{G}_{i}\textbf{x}^{d_{i}})^{T})^{T} \in \mathbb{R}^{d_{i}\cdot m},
$$
where $dim(\mathcal{L}_{i}) = d_{i} \leq d$ and $\textbf{x}^{1},...,\textbf{x}^{d_{i}}$ stands for some fixed basis of $\mathcal{L}_{i}$.

Denote by $f_{i}^{\prime}$ the structured counterpart of $f_{i}$, where $\textbf{G}_{i}$ is replaced by the structured spinner (for which vector $\textbf{r}$ is either learned or random). We will show that $f_{i}^{\prime}s$ ``resemble'' $f_{i}s$ distribution-wise.
Surprisingly, we will show it under very weak conditions regarding $f_{i}s$, In particular, they can be nondifferentiable, even non-continuous.

Note that the above setting covers a wide range of machine learning algorithms. In particular:

\begin{remark}
In the kernel approximation setting with random feature maps one can match each pair of vectors $\textbf{x},\textbf{y} \in \mathcal{X}$ to a different $f=f_{\textbf{x},\textbf{y}}$.
Each $f$ computes the approximate value of the kernel for vectors $\textbf{x}$ and $\textbf{y}$.
Thus in that scenario $s = {|\mathcal{X}| \choose 2}$ and $d=2$ (since one can take: $\mathcal{L}_{f(\textbf{x},\textbf{y})} = span(\textbf{x},\textbf{y})$).
\end{remark}

\begin{remark}
In the vector quantization algorithms using random projection trees one can take $s=\nobreak 1$ (the algorithm $\mathcal{A}$ itself is a function $f$ outputting the partitioning of space into cells) and $d = d_{intrinsic}$, where $d_{intrinsic}$ is an intrinsic dimensionality of a given dataset $\mathcal{X}$ (random projection trees are often used if $d_{intrinsic} \ll n$).
\end{remark}

\subsection{Random setting}

We need the following definition.

\begin{definition}
A set $\mathcal{S}$ is $b$-convex if it is a union of at most $b$ pairwise disjoint convex sets.
\end{definition}

Fix a funcion $f_{i}: \mathbb{R}^{d_{i} \cdot m} \rightarrow \mathcal{V}$, for some domain $\mathcal{V}$.
Our main result states that for any $\mathcal{S} \subseteq \mathcal{V}$ such that $f_{i}^{-1}(\mathcal{S})$  is measurable and $b$-convex for $b$ not too large, the probability that $f_{i}(\textbf{q}_{f_{i}})$ belongs to $\mathcal{S}$ is close to the probability that $f^{\prime}_{i}(\textbf{q}_{f^{\prime}_{i}})$ belongs to $\mathcal{S}$.

\begin{theorem}[structured random setting]
\label{main_struct_theorem}
Let $\mathcal{A}$ be a randomized algorithm using unstructured Gaussian matrices $\textbf{G}$ and let $s,d$ and $f_{i}$s be as at the beginning of the section.
Replace the unstructured matrix $\textbf{G}$ by one of 
structured spinners defined in Section \ref{sec:model} with blocks of $m$ rows each. 
Then for $n$ large enough, $\epsilon = o_{md}(1)$ and fixed $f_{i}$ 
with probability $p_{succ}$ at least:
\vspace{-0.1in}
\begin{multline}
\label{imp_p}
\hspace{-0.1in}1 - 2p(n)d 
- 2{md \choose 2}e^{-\Omega(\min(\frac{\epsilon^{2}n^{2}}{K^{4}\Lambda_{F}^{2}\delta^{4}(n)}, \frac{\epsilon n}{K^{2}\Lambda_{2} \delta^{2}(n)}))}
\end{multline}
with respect to the random choices of $\textbf{M}_{1}$ and $\textbf{M}_{2}$
the following holds for any $\mathcal{S}$ such that $f^{-1}_{i}(\mathcal{S})$ is measurable and $b$-convex:
$$|\mathbb{P}[f_{i}(\textbf{q}_{f_{i}}) \in \mathcal{S}]-\mathbb{P}[f^{\prime}_{i}(\textbf{q}_{f^{\prime}_{i}}) \in \mathcal{S}]| \leq b \eta,$$
where the the probabilities in the last formula are with respect to the random choice of $\textbf{M}_{3}$, $\eta = \frac{\delta^{3}(n)}{n^{\frac{2}{5}}}$, and $\delta(n), p(n),K, \Lambda_{F}, \Lambda_{2}$ are as in the definition of  structured spinners from Section \ref{sec:model}.
\end{theorem}

\begin{remark}
The theorem does not require any strong regularity conditions regarding $f_{i}s$ (such as differentiability or even continuity). In practice, $b$ is often a small constant. For instance, for the angular kernel approximation where $f_{i}$s are non-continuous and for $\mathcal{S}$-singletons, we can take $b=1$ (see Supplement).
\end{remark}

Now let us think of $f_{i}$ and $f_{i}^{\prime}$
as random variables, where randomness is generated by vectors $\textbf{q}_{f_{i}}$ and $\textbf{q}_{f^{\prime}_{i}}$ respectively.
Then, from Theorem \ref{main_struct_theorem}, we get:

\begin{theorem}
\label{convex_theorem}
Denote by $F_{X}$ the cdf of the random variable $X$ and by $\phi_{X}$
its characteristic function.
If $f_{i}$ is convex or concave in respect to $\textbf{q}_{f_{i}}$,
then for every $t$ the following holds: $|F_{f_{i}}(t) - F_{f^{\prime}_{i}}(t)| = O(\frac{\delta^{3}(n)}{n^{\frac{2}{5}}})$.
Furthermore, if $f_{i}$ is bounded then:
$|\phi_{f_{i}}(t) - \phi_{f^{\prime}_{i}}(t)| = O(\frac{\delta^{3}(n)}{n^{\frac{2}{5}}})$. 
\end{theorem}

Theorem \ref{main_struct_theorem} implies strong accuracy guarantees for the specific structured spinners. As a corollary we get:

\begin{theorem}
\label{corollary_theorem}
Under assumptions from Theorem~\ref{main_struct_theorem}
the probability $p_{succ}$ from Theorem~\ref{main_struct_theorem} reduces to:
$1 - 4ne^{-\frac{\log^{2}(n)}{8}}d - 2{md \choose 2} e^{-\Omega(\frac{\epsilon^{2}n}{\log^{4}(n)})}$
for the structured matrices $\sqrt{n}\textbf{HD}_{3}\textbf{HD}_{2}\textbf{HD}_{1}$, $\sqrt{n}\textbf{HD}_{g_{1},...,g_{n}}\textbf{HD}_{2}\textbf{HD}_{1}$ as well as for the structured matrices of the form $\textbf{G}_{struct}\textbf{D}_{2}\textbf{HD}_{1}$, where $\textbf{G}_{struct}$ is Gaussian circulant, Gaussian Toeplitz or Gaussian Hankel matrix.
\end{theorem}


As a corollary of Theorem~\ref{corollary_theorem}, we obtain the following result showing the effectiveness of the cross-polytope LSH with structured matrices $\textbf{HD}_{3}\textbf{HD}_{2}\textbf{HD}_{1}$ that was only heuristically confirmed before~\cite{indyk2015}.

\begin{theorem}
\label{hopefully_last_theorem}
Let $\textbf{x},\textbf{y} \in \mathbb{R}^{n}$ be two unit $L_{2}$-norm vectors. 
Let $\textbf{v}_{\textbf{x}, \textbf{y}}$ be the vector indexed by all $(2m)^{2}$ ordered pairs of canonical directions $(\pm \textbf{e}_{i},\pm \textbf{e}_{j})$, where the value of the entry indexed by $(\textbf{u},\textbf{w})$ is the probability that: $h(\textbf{x}) = \textbf{u}$ and $h(\textbf{y}) = \textbf{w}$, and $h(\textbf{v})$ stands for the hash of $\textbf{v}$. Then with probability at least: 
$p_{success} =  1 - 8ne^{-\frac{\log^{2}(n)}{8}} - 2{2m \choose 2} e^{-\Omega(\frac{\epsilon^{2}n}{\log^{4}(n)})}$
the version of the stochastic vector $\textbf{v}^{1}_{\textbf{x},\textbf{y}}$
for the unstructured Gaussian matrix $\textbf{G}$ and its structured counterpart $\textbf{v}^{2}_{\textbf{x},\textbf{y}}$ for the matrix $\textbf{HD}_{3}\textbf{HD}_{2}\textbf{HD}_{1}$ satisfy: 
$\|\textbf{v}^{1}_{\textbf{x},\textbf{y}} - \textbf{v}^{2}_{\textbf{x},\textbf{y}}\|_{\infty} \leq \log^{3}(n)n^{-\frac{2}{5}} + c \epsilon$,
for $n$ large enough, where $c > 0$ is a universal constant.
The probability above is taken with respect to random choices of $\textbf{D}_{1}$ and $\textbf{D}_{2}$.
\end{theorem}

For angles in the range $[0, \frac{\pi}{3}]$
the result above leads to the same asymptotics of the probabilities of collisions as these in Theorem 1 of \cite{indyk2015} given for the unstructured cross-polytope LSH.

The proof for the discrete structured setting applies Berry-Esseen-type results for random vectors (details are in the Supplement) showing that for $n$ large enough $\pm 1$ random vectors $\textbf{r}$ act similarly to Gaussian vectors.

\subsection{Adaptive setting}

The following theorem explains that structured spinners can be used to replace unstructured fully connected neural network layers performing dimensionality reduction (such as hidden layers in certain autoencoders) provided that input data has low intrinsic dimensionality. These theoretical findings were confirmed in experiments that will be presented in the next section.
We will use notation from Theorem \ref{main_struct_theorem}.

\begin{theorem}
\label{neural_theorem}
Consider a matrix $\textbf{M} \in \mathbb{R}^{m \times n}$ encoding the weights of connections between a layer $l_{0}$ of size $n$ and a layer $l_{1}$ of size $m$ in some learned unstructured neural network model. Assume that the input to layer $l_{0}$ is taken from the $d$-dimensional space $\mathcal{L}$ (although potentially embedded in a much higher dimensional space).
Then with probability at least 
\vspace{-0.1in}
\begin{multline}
\label{imp_p}
\hspace{-0.1in}1 - 2p(n)d 
- 2{md \choose 2}e^{-\Omega(\min(\frac{t^{2}n^{2}}{K^{4}\Lambda_{F}^{2}\delta^{4}(n)}, \frac{t n}{K^{2}\Lambda_{2} \delta^{2}(n)}))}
\end{multline}
for $t = \frac{1}{md}$ and
with respect to random choices of $\textbf{M}_{1}$ and $\textbf{M}_{2}$, there exists a vector $\textbf{r}$ defining $\textbf{M}_{3}$ (see: definition of the structured spinner)
such that the structured spinner $\textbf{M}^{struct} = \textbf{M}_{3}\textbf{M}_{2}\textbf{M}_{1}$ 
equals to $\textbf{M}$ on $\mathcal{L}$.
\end{theorem}

\section{Experiments}
\label{sec:experiments}

In this section we consider a wide range of different applications of structured spinners: locality-sensitive hashing, kernel approximations, and finally neural networks. Experiments with Newton sketches are deferred to the Supplement. Experiments were conducted using Python. In particular, NumPy is linked against a highly optimized BLAS library (Intel MKL). Fast Fourier Transform is performed using numpy.fft and Fast Hadamard Transform is using ffht from~\cite{indyk2015}. To have a fair comparison, we have set up: $\mathrm{OMP\_NUM\_THREADS}=1$ so that every experiment is done on a single thread. Every parameter of the structured spinner matrix is computed in advance, such that obtained speedups take only matrix-vector products into account. 
All figures should be read in color.

\subsection{Locality-Sensitive Hashing (LSH)} 

\begin{figure}[t!]
\vspace{-0.15in}
\centering
\begin{subfigure}[b]{\linewidth}
\centering
\includegraphics[width=\columnwidth]{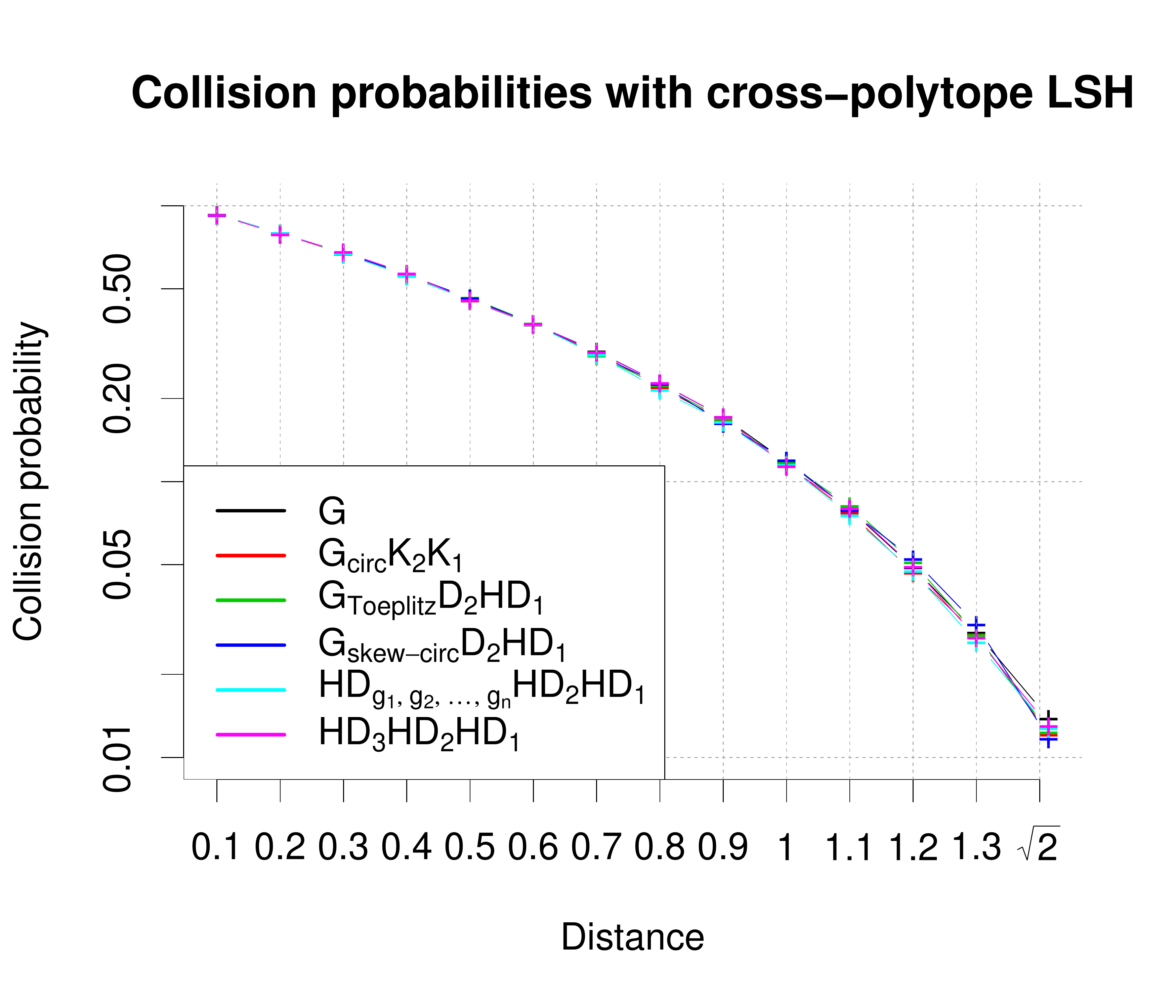}
\label{fig:collision_a}
\end{subfigure}%
\\
\vspace{-0.7in}
\begin{subfigure}[b]{\linewidth}
\centering
\includegraphics[width=\columnwidth]{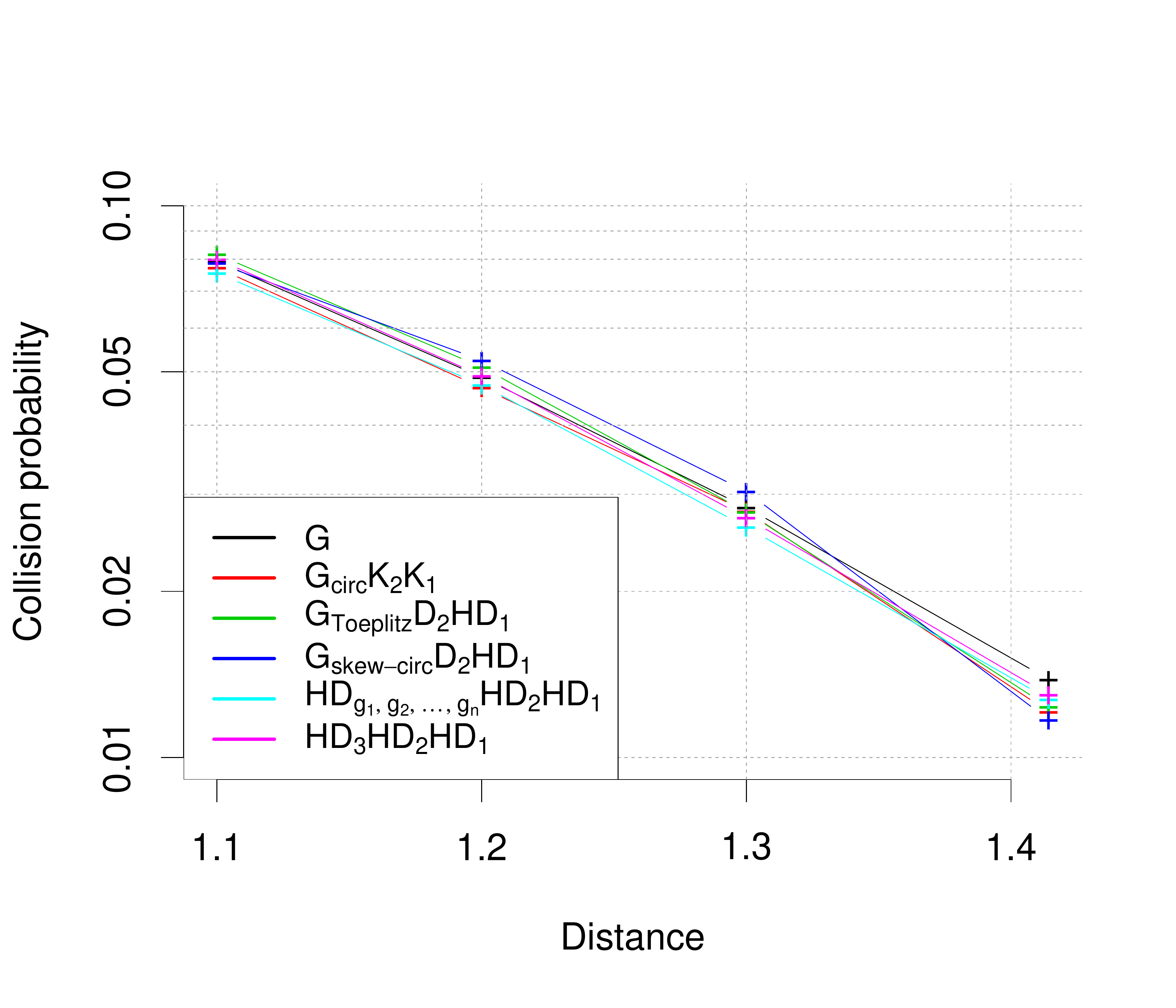}
\label{fig:collision_b}
\end{subfigure}%
\vspace{-0.4in}
\caption{Cross-polytope LSH - collision  probabilities. (bottom) A zoom on higher distances enables to distinguish the curves which are almost superposed. \vspace{-0.2in}} \label{fig:collision}
\end{figure} 

In the first experiment, we consider cross-polytope LSH. In Figure \ref{fig:collision}, we compare collision probabilities for the low dimensional case ($n = 256$), where for each interval, collision probability has been computed for $20 000$ points. Results are shown for one hash function (averaged over $100$ runs). We report results for a random $256 \times 64$ Gaussian matrix $\textbf{G}$ and five other types of matrices from a family of structured spinners (descending order of number of parameters): $\textbf{G}_{circ}\textbf{K}_{2}\textbf{K}_{1}$, $\textbf{G}_{Toeplitz}\textbf{D}_{2}\textbf{HD}_{1}$, $\textbf{G}_{skew-circ}\textbf{D}_{2}\textbf{HD}_{1}$, $\textbf{HD}_{g_{1},...,g_{n}}\textbf{HD}_{2}\textbf{HD}_{1}$, and $\textbf{HD}_{3}\textbf{HD}_{2}\textbf{HD}_{1}$, where $\textbf{K}_i$, $\textbf{G}_{Toeplitz}$, and $\textbf{G}_{skew-circ}$ are respectively a Kronecker matrix with discrete entries, Gaussian Toeplitz and Gaussian skew-circulant matrices. 

All matrices from the family of structured spinners show high collision probabilities for small distances and low ones for large distances. As theoretically predicted, structured spinners do not lead to accuracy losses.
All considered matrices give almost identical results.

\subsection{Kernel approximation}

\begin{figure*}[htp!]
\centering
\includegraphics[width=0.49\columnwidth]{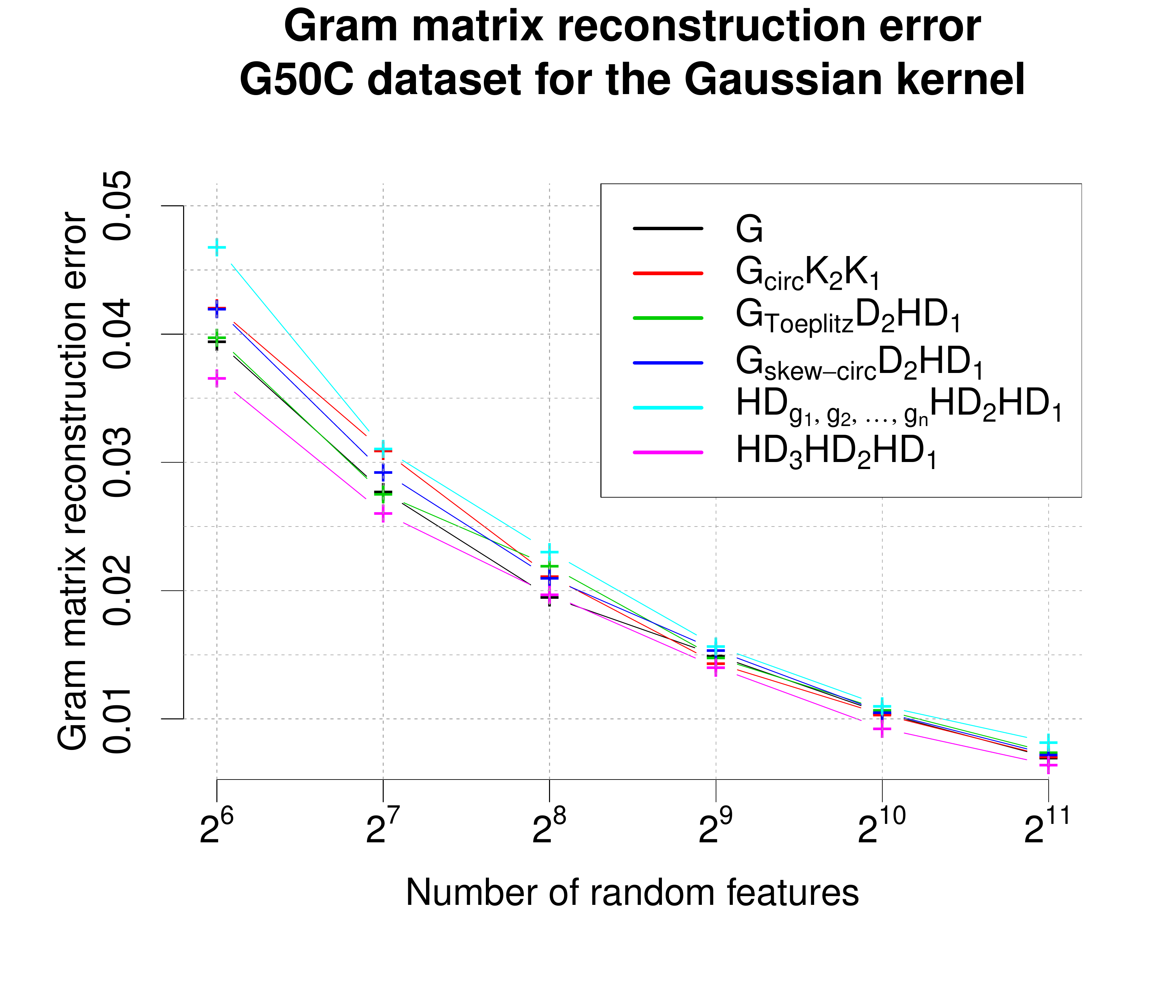} 
\includegraphics[width=0.49\columnwidth]{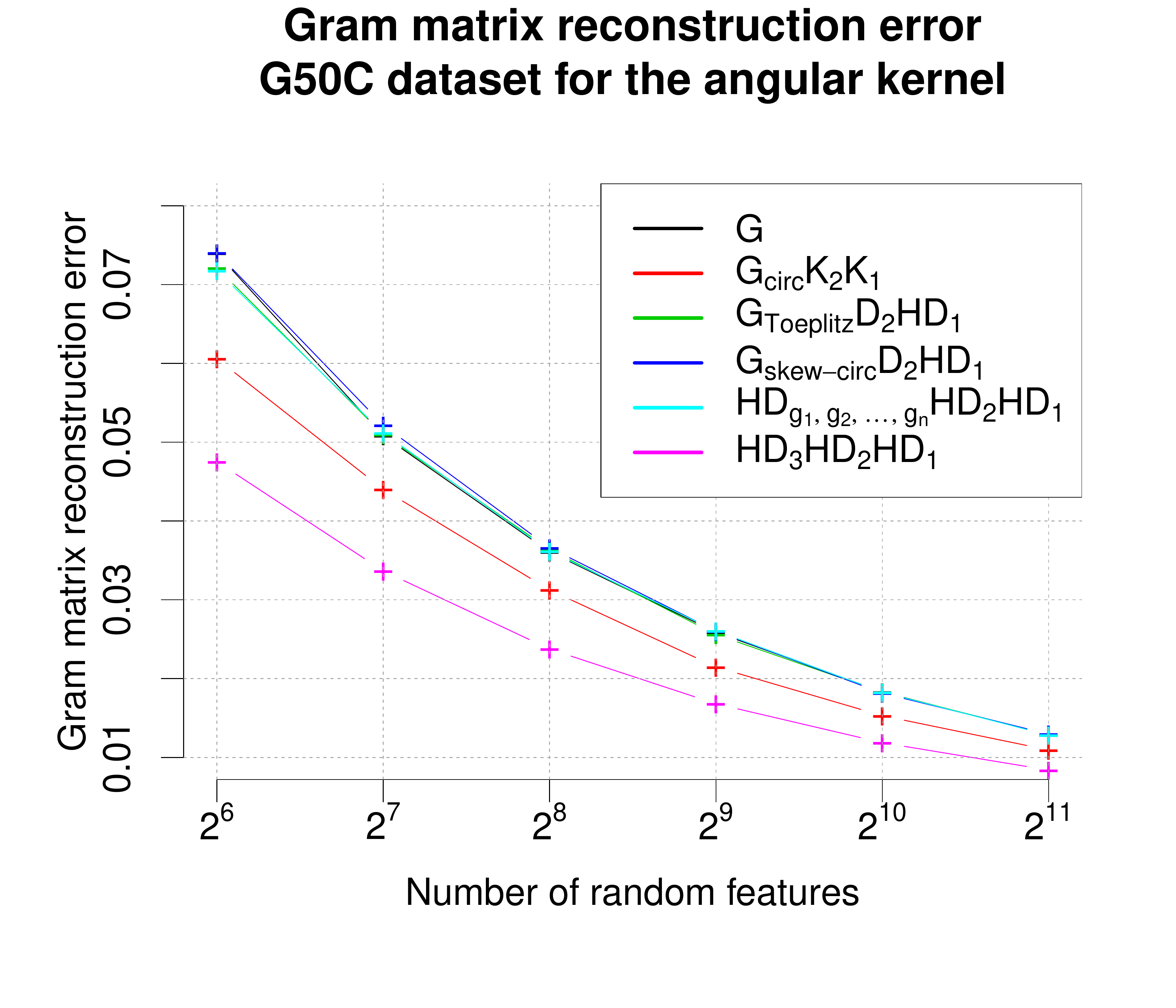}
\vspace{-0.4in}
\caption{Accuracy of random feature map kernel approximation for the G50C dataset.\vspace{-0.05in}}
\label{kernel_approx}
\end{figure*}
\begin{table*}[t!]
\vspace{-0.07in}
\begin{center}
\begin{tabular}{|c||c|c|c|c|c|c|c|}
\hline
{\bf MATRIX DIM.}  & {\bf $2^{9}$} & {\bf $2^{10}$} & {\bf $2^{11}$} & {\bf $2^{12}$} & {\bf $2^{13}$} & {\bf $2^{14}$} & {\bf $2^{15}$} \\
\hline 
\hline
$\textbf{G}_{Toeplitz}\textbf{D}_{2}\textbf{HD}_{1}$  & x1.4      & x3.4       & x6.4       & x12.9   & x28.0       & x42.3       & x89.6       \\
\hline
$\textbf{G}_{skew-circ}\textbf{D}_{2}\textbf{HD}_{1}$ & x1.5      & x3.6       & x6.8       & x14.9    & x31.2       & x49.7       & x96.5       \\
\hline
$\textbf{HD}_{g_{1},...,g_{n}}\textbf{HD}_{2}\textbf{HD}_{1}$ & x2.3      & x6.0       & x13.8       & x31.5     & x75.7       & x137.0      & x308.8       \\
\hline
$\textbf{HD}_{3}\textbf{HD}_{2}\textbf{HD}_{1}$ & x2.2      & x6.0       & x14.1       & x33.3     & x74.3       & x140.4       & x316.8           \\
\hline
\end{tabular}
\end{center}
\vspace{-0.18in}
\caption{Speedups for Gaussian kernel approximation via structured spinners. \vspace{-0.1in}} 
\label{tab:speedupskernel}
\end{table*}

\begin{table*}[t!]
\vspace{-0.05in}
\begin{center}
\begin{tabular}{|c||c|c|c|c|c|c|c|c|c|}
\hline
{\bf h}  & {\bf $2^{4}$} & {\bf $2^{5}$} & {\bf $2^{6}$} & {\bf $2^{7}$} & {\bf $2^{8}$} & {\bf $2^{9}$} & {\bf $2^{10}$} & {\bf $2^{11}$} & {\bf $2^{12}$}\\
\hline
\hline
\textbf{unstructured} & $42.9$ & $51.9$ & $72.7$ & $99.9$ & $163.9$ & $350.5$ & $716.7$ & $1271.5$ & $2317.4$           \\
\hline
$\textbf{HD}_{3}\textbf{HD}_{2}\textbf{HD}_{1}$ & $109.2$ & $121.3$ & $109.7$ & $114.2$ & $117.4$ & $123.9$ & $130.6$ & $214.3$ & $389.8$           \\ 
\hline
\end{tabular}
\end{center}
\vspace{-0.2in}
\caption{Running time (in $[\mu s]$) for the MLP - unstructured matrices vs structured spinners.\vspace{-0.2in}} 
\label{tab:nnruntime}
\end{table*}
\begin{figure*}[t!]
\vspace{-0.12in}
\centering
\includegraphics[width=0.49\columnwidth]{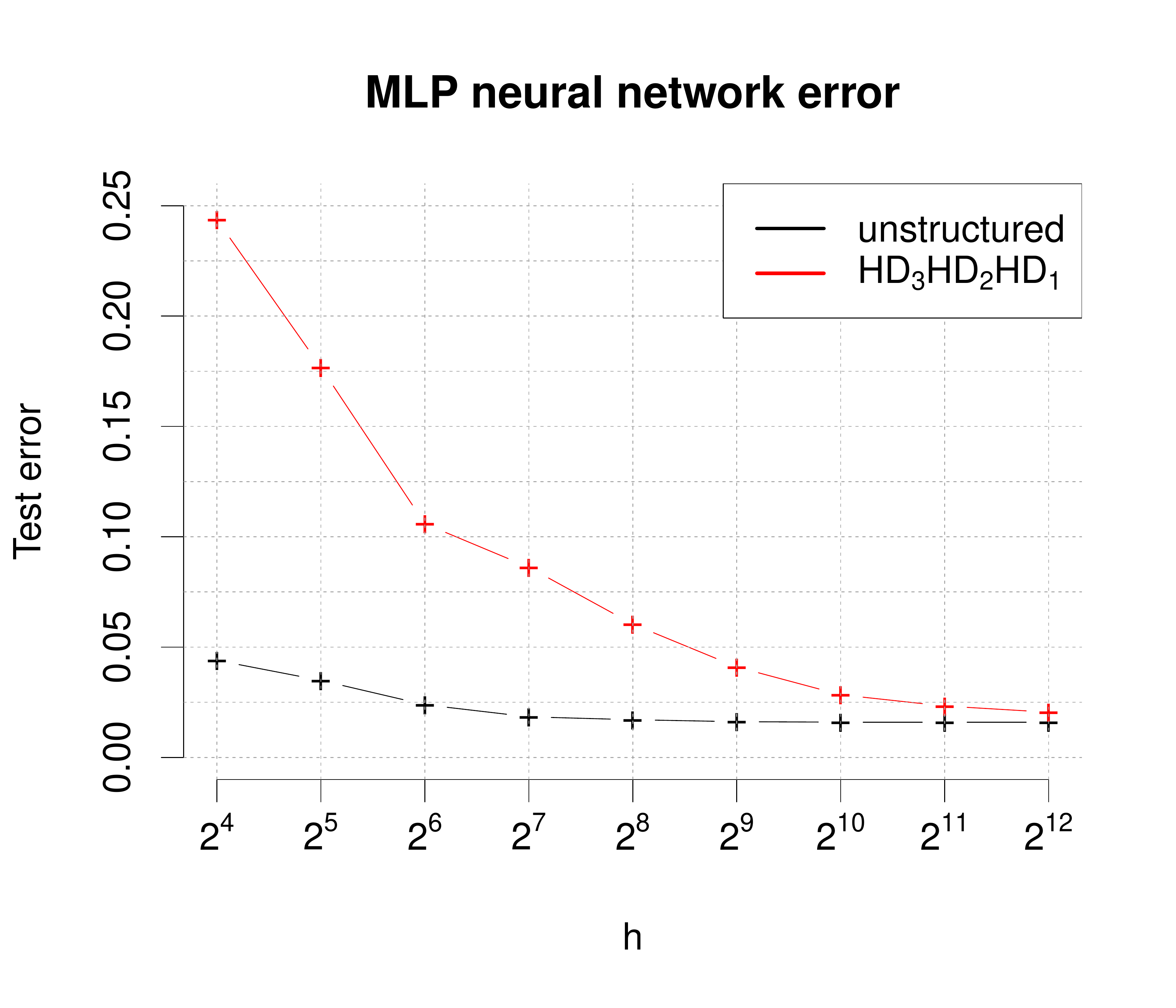}
\vspace{-0.3in}
\includegraphics[width=0.49\columnwidth]{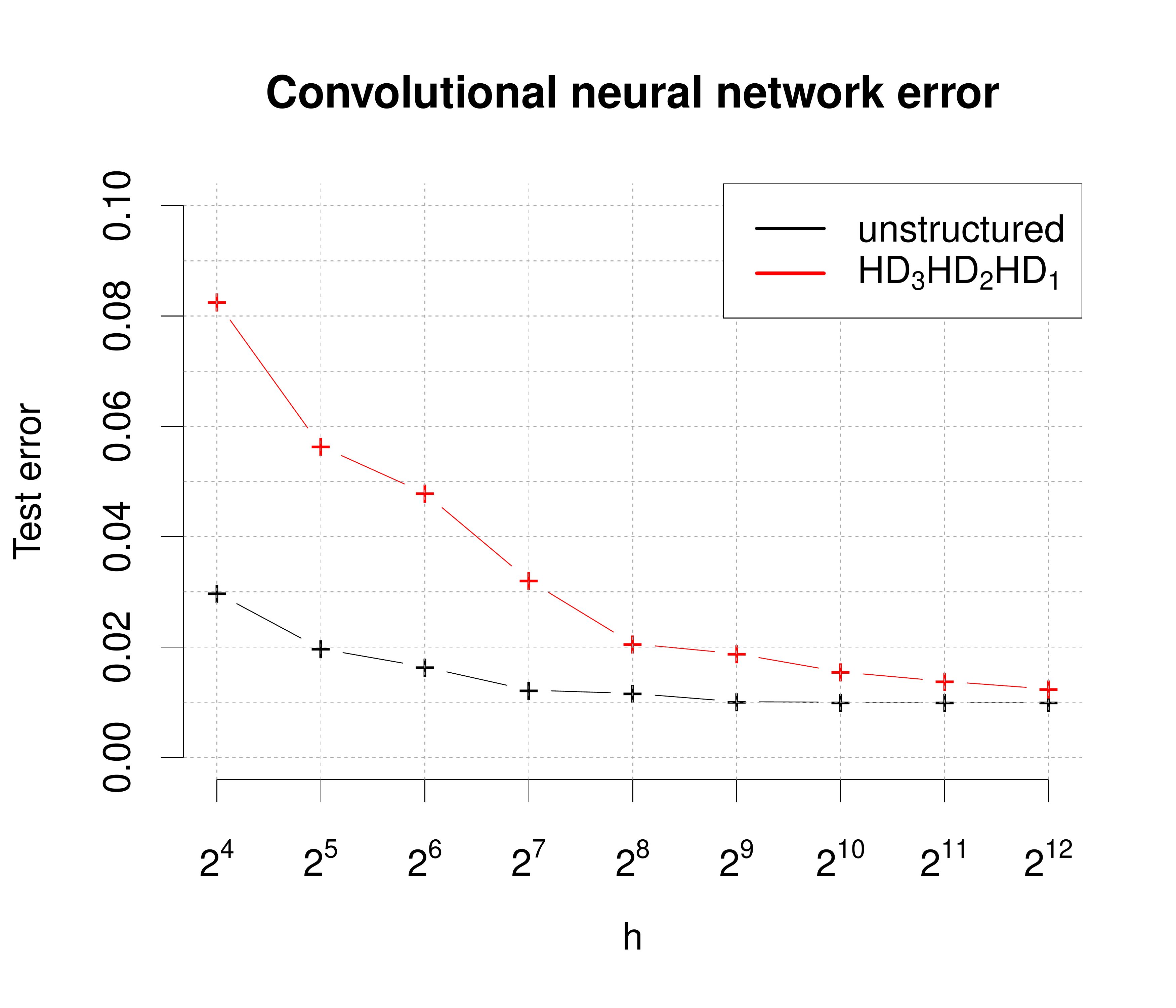}
\vspace{+0.3in}
\caption{Test error for MLP (top) and convolutional network (bottom).\vspace{-0.2in}}
\label{fig:nn}
\vspace{-0.3in}
\end{figure*}

In the second experiment, we approximate the Gaussian and angular kernels using Random Fourier features. The Gaussian random matrix (with i.i.d. Gaussian entries) can be used to sample random Fourier features with a specified $\sigma$. This Gaussian random matrix is replaced with specific matrices from a family of structured spinners for Gaussian and angular kernels. The obtained feature maps are compared. To test the quality of the structured kernels' approximations, we compute Gram-matrix reconstruction error as in \cite{chor_sind_2016} : $ \frac{||\textbf{K} - \tilde{\textbf{K}} ||_{F}}{||\textbf{K}||_{F}} $, where $\textbf{K}, \tilde{\textbf{K}}$ are respectively the exact and approximate Gram-matrices, as a function of the number of random features. When number of random features $k$ is greater than data dimensionality $n$, we apply
block-mechanism described in \ref{subsec:block}.

For the Gaussian kernel,  $\textbf{K}_{ij} = e^{ \frac{ -||\textbf{x}_{i} - \textbf{x}_{j}||_{2}^{2}}{2 \sigma ^{2}}}$ and for the angular kernel,  $\textbf{K}_{ij}= 1 - \frac{\theta}{\pi} $ with $\theta = cos^{-1}(\frac{\textbf{x}_{i}^T \textbf{x}_{j}}{||\textbf{x}_{i}||_{2} ||\textbf{x}_{j}||_{2}})$. For the approximation, $\tilde{\textbf{K}}_{i,j} = \frac{1}{ \sqrt{d'}}s(\textbf{A} \textbf{x}_{i})^{T} \frac{1}{ \sqrt{d'}} s(\textbf{A} \textbf{x}_{j})$ where $s(x) = e^{ \frac{-i x}{\sigma}}$ and $\tilde{\textbf{K}}_{i,j} = 1 - \frac{ \operatorname{d_H}(  s(\textbf{A} \textbf{x}_{i}), s(\textbf{A} \textbf{x}_{j}))}{d'} $ where $s(x) = \operatorname{sign}(x)$ respectively. In both cases, function $s$ is applied pointwise. $d_H$ stands for the Hamming distance and $x_i$, $x_j$ are points from the dataset.

We used two datasets: G50C ($550$ points, $n = 50$) and USPST (test set, $2007$ points, $n = 256$). The results for the USPST dataset are given in the Supplement. For Gaussian kernel, bandwidth $\sigma$ is set to $17.4734$ for G50C and to $9.4338$ for USPST. The choice of $\sigma$ comes from \cite{chor_sind_2016} in order to have comparable results. 
The results are averaged over $10$ runs and the following matrices have been tested: Gaussian random matrix $\textbf{G}$,  $\textbf{G}_{circ}\textbf{K}_{2}\textbf{K}_{1}$, $\textbf{G}_{Toeplitz}\textbf{D}_{2}\textbf{HD}_{1}$, $\textbf{G}_{skew-circ}\textbf{D}_{2}\textbf{HD}_{1}$, $\textbf{HD}_{g_{1},...,g_{n}}\textbf{HD}_{2}\textbf{HD}_{1}$ and $\textbf{HD}_{3}\textbf{HD}_{2}\textbf{HD}_{1}$.

Figure \ref{kernel_approx} shows results for the G50C dataset. In case of G50C dataset, for both kernels, all matrices from the family of structured spinners perform similarly to a random Gaussian matrix. 
$\textbf{HD}_{3}\textbf{HD}_{2}\textbf{HD}_{1}$ performs better than all other matrices for a wide range of sizes of random feature maps. In case of USPST dataset (see: Supplement), for both kernels, all matrices from the family of structured spinners again perform similarly to a random Gaussian matrix (except $\textbf{G}_{circ}\textbf{K}_{2}\textbf{K}_{1}$ which gives relatively poor results) and $\textbf{HD}_{3}\textbf{HD}_{2}\textbf{HD}_{1}$ is giving the best results. Finally, the efficiency of structured spinners does not depend on the dataset.

Table \ref{tab:speedupskernel} shows substantial speedups
obtained by the structured spinner matrices. The speedups are computed as $\mbox{time}(\textbf{G})/\mbox{time}(\textbf{T})$, where $\mbox{time}(\textbf{G})$ and $\mbox{time}(\textbf{T})$ are the runtimes for respectively a random Gaussian matrix and a structured spinner matrix.

\subsection{Neural networks} 

Finally, we performed experiments with neural networks using two different network architectures. The first one is a fully-connected network with two fully connected layers (we call it MLP), where we refer to the size of the hidden layer as $h$, and the second one is a convolutional network with following architecture:
\begin{itemize}
\vspace{-0.15in}
\item Convolution layer with filter size $5 \times 5$, 4 feature maps + ReLU + Max Pooling (region $2 \times 2$ and step $2 \times 2$)
\vspace{-0.1in}
\item Convolution layer with filter size $5 \times 5$, 6 feature maps + ReLU + Max Pooling (region $2 \times 2$ and step $2 \times 2$)
\vspace{-0.08in}
\item Fully-connected layer ($h$ outputs) + ReLU
\vspace{-0.08in}
\item Fully-connected layer ($10$ outputs)
\vspace{-0.08in}
\item LogSoftMax.
\vspace{-0.15in}
\end{itemize}

Experiments were performed on the MNIST data set. In both experiments, we re-parametrized each matrix of weights of fully connected layers with a structured $\textbf{HD}_{3}\textbf{HD}_{2}\textbf{HD}_{1}$ matrix from a family of structured spinners. We compare this setting with the case where the unstructured parameter matrix is used. Note that in case when we use $\textbf{HD}_{3}\textbf{HD}_{2}\textbf{HD}_{1}$ only linear number of parameters is learned (the Hadamard matrix is deterministic and even does not need to be explicitly stored, instead Walsh-Hadamard transform is used).
Thus the network has significantly less parameters than in the unstructured case, e.g. for the MLP network we have $\mathcal{O}(h)$ instead of $\mathcal{O}(\text{input size} \times h)$ parameters.

In Figure~\ref{fig:nn} and Table~\ref{tab:nnruntime} we compare respectively the test error and running time of the unstructured and structured approaches. Figure~\ref{fig:nn} shows that for large enough $h$, neural networks with structured spinners achieve similar performance to those with unstructured projections, while at the same time using structured spinners lead to significant computational savings as shown in Table~\ref{tab:nnruntime}.
As mentioned before, the $\textbf{HD}_{3}\textbf{HD}_{2}\textbf{HD}_{1}$-neural network is a simpler construction than the Deep Friend Convnet, however one can replace it with any structured spinner to obtain compressed neural network architecture of a good capacity.

\newpage
\bibliographystyle{apalike}
\bibliography{StructuredSpinners}

\newpage
\clearpage

\toptitlebar 
{\Large \bf  \centering{Structured adaptive and random spinners for fast machine learning computations\\(Supplementary Material)} \par}
\bottomtitlebar

In the Supplementary material we prove all theorems presented in the main body of the paper.

\subsection{Structured machine learning algorithms with \textit{Structured Spinners}}

We prove now Lemma \ref{simple_lemma}, Remark \ref{balanceness_remark},
as well as Theorem \ref{main_struct_theorem} and Theorem \ref{corollary_theorem}.

\subsubsection{Proof of Remark \ref{balanceness_remark}}

This result first appeared in \cite{ailon2006approximate}. The following proof was given in \cite{chor_sind_2016}, we repeat it here for completeness.
We will use the following standard concentration result.

\begin{lemma}(Azuma's Inequality)
Let $X_{1},...,X_{n}$ be a martingale and assume that $-\alpha_{i} \leq X_{i} \leq \beta_{i}$ for some positive constants $\alpha_{1},...,\alpha_{n}, \beta_{1},...,\beta_{n}$. 
Denote $X = \sum_{i=1}^{n} X_{i}$.
Then the following is true:
\begin{equation}
\mathbb{P}[|X - \mathbb{E}[X]| > a] \leq 2e^{-\frac{a^{2}}{2\sum_{i=1}^{n}(\alpha_{i} + \beta_{i})^{2}}}
\end{equation}
\end{lemma}

\begin{proof}
Denote by $\tilde{\textbf{x}}^{j}$ an image of $\textbf{x}^{j}$ under transformation $\textbf{HD}$. Note that the $i^{th}$ dimension of $\tilde{\textbf{x}}^{j}$
is given by the formula: $\tilde{x}^{j}_{i} = h_{i,1}x^{j}_{1} + ... + h_{i,n}x^{j,n}$,
where $h_{l,u}$ stands for the $l^{th}$ element of the $u^{th}$ column of the randomized
Hadamard matrix $\textbf{HD}$.
First, we use Azuma's Inequality to find an upper bound on the probability
that $|\tilde{x}^{j}_{i}| > a$, where $a=\frac{\log(n)}{\sqrt{n}}$.
By Azuma's Inequality, we have:
\begin{equation}
\mathbb{P}[|h_{i,1}x^{j}_{1} + ... + h_{i,n}x^{j,n}| \geq a] \leq 2e^{-\frac{\log^{2}(n)}{8}}.
\end{equation}
We use: $\alpha_{i} = \beta_{i} = \frac{1}{\sqrt{n}}$.
Now we take the union bound over all $n$ dimensions and the proof is completed.
\end{proof}

\subsubsection{\textit{Structured Spinners}-equivalent definition}
We will introduce here an equivalent definition of the model of structured spinners that is more technical (thus we did not give it in the main body of the paper), yet more convenient to work with in the proofs.

Note that from the definition of structured spinners we can conclude that each structured matrix $\textbf{G}_{struct} \in \mathbb{R}^{n \times n}$ from the family of structured spinners is a product of three main structured blocks, i.e.:
\begin{equation}
\textbf{G}_{struct} = \textbf{B}_{3}\textbf{B}_{2}\textbf{B}_{1},
\end{equation}

where matrices $\textbf{B}_{1},\textbf{B}_{2},\textbf{B}_{3}$ satisfy two conditions that we give below.

\begin{framed}
\textbf{Condition 1:} Matrices: $\textbf{B}_{1}$ and $\textbf{B}_{2}\textbf{B}_{1}$ are $(\delta(n),p(n))$-balanced isometries. \\
\textbf{Condition 2:} Pair of matrices $(\textbf{B}_{2},\textbf{B}_{3})$
is $(K,\Lambda_{F}, \Lambda_{2})$-random.
\end{framed}

Below we give the definition of $(K, \Lambda_{F}, \Lambda_{2})$-randomness.

\begin{definition}[$(K, \Lambda_{F}, \Lambda_{2})$-randomness]
A pair of matrices $(\textbf{Y},\textbf{Z}) \in \mathbb{R}^{n \times n} \times \mathbb{R}^{n \times n}$ is $(K, \Lambda_{F}, \Lambda_{2})$-random
if there exists $\textbf{r} \in \mathbb{R}^{k}$, and
a set of linear isometries $\phi = \{\phi_{1},...,\phi_{n}\}$, 
where $\phi_i : \mathbb{R}^{n} \rightarrow \mathbb{R}^{k}$, such that:
\begin{itemize}
\item $\textbf{r}$ is either a $\pm 1$-vector with i.i.d. entries
      or Gaussian with identity covariance matrix,
\item for every $\textbf{x} \in \mathbb{R}^{n}$ the $j^{th}$ element $(\textbf{Zx})_{j}$ of $\textbf{Zx}$ is of the form: $\textbf{r}^{T} \cdot \phi_{j}(\textbf{x})$,     
\item there exists a set of i.i.d. sub-Gaussian random variables $\{\rho_{1},...,\rho_{n}\}$ with sub-Gaussian norm at most $K$, mean $0$, the same second moments and a $(\Lambda_{F},\Lambda_{2})$-smooth set of matrices $\{\textbf{W}^{i}\}_{i=1,...,n}$ such that for every $\textbf{x} = (x_{1},...,x_{n})^{T}$, we have: $\phi_{i}(\textbf{Y}\textbf{x}) = \textbf{W}^{i} (\rho_{1}x_{1},...,\rho_{n}x_{n})^{T}$.
\end{itemize}
\end{definition}

\subsubsection{Proof of Lemma \ref{simple_lemma}}

\begin{proof}
Let us first assume the $\textbf{G}_{circ}\textbf{D}_{2}\textbf{HD}_{1}$-setting
(analysis for Toeplitz Gaussian or Hankel Gaussian is completely analogous).
In that setting, it is easy to see that one can take $\textbf{r}$ to be a Gaussian vector (this vector corresponds to the first row of $\textbf{G}_{circ}$). Furthermore linear mappings $\phi_{i}$ are defined as: $\phi_{i}((x_{0},x_{1},...,x_{n-1})^{T}) = (x_{n-i},x_{n-i+1},...,x_{i-1})^{T}$, where operations on indices are modulo $n$.
The value of $\delta(n)$ and $p(n)$ come from the fact that matrix $\textbf{HD}_{1}$ is used as a $(\delta(n),p(n))$-balanced matrix and from Remark \ref{balanceness_remark}.
In that setting, sequence $(\rho_{1},...,\rho_{n})$ is discrete and corresponds to the diagonal of $\textbf{D}_{2}$.
Thus we have: $K = 1$. To calculate $\Lambda_{F}$ and $\Lambda_{2}$, note first that matrix $\textbf{W}^{1}$ is defined as $\textbf{I}$ and subsequent $\textbf{W}^{i}$s are given as circulant shifts of the previous ones (i.e. each row is a circulant shift of the previous row). That observation comes directly from the circulant structure of $\textbf{G}_{circ}$. Thus we have: $\Lambda_{F} = O(\sqrt{n})$ and $\Lambda_{2} = O(1)$. The former is true since each $\textbf{A}^{i,j}$ has $O(n)$ nonzero entries and these are all $1$s. The latter is true since each nontrivial $\textbf{A}^{i,j}$ in that setting is an isometry (this is straightforward from the definition of $\{\textbf{W}^{i}\}_{i=1,...,n}$). Finally, all other conditions regarding $\textbf{W}^{i}$-matrices are clearly satisfied (each column of each $\textbf{W}^{i}$ has unit $L_{2}$ norm and corresponding columns from different $\textbf{W}^{i}$ and $\textbf{W}^{j}$ are clearly orthogonal).

Now let us consider the setting, where the structured matrix is of the form: $\sqrt{n}\textbf{HD}_{3}\textbf{HD}_{2}\textbf{HD}_{1}$.
In that case, $\textbf{r}$ corresponds to a discrete vector (namely, the diagonal of $\textbf{D}_{3}$).
Linear mappings $\phi_{i}$ are defined as:
$\phi_{i}((x_{1},...,x_{n})^{T}) = (\sqrt{n}h_{i,1}x_{1},...,\sqrt{n}h_{i,n}x_{n})^{T}$, where $(h_{i,1},...,h_{i,n})^{T}$ is the $i^{th}$ row of $\textbf{H}$.
One can also notice that the set $\{\textbf{W}^{i}\}_{i=1,...,n}$ is defined as: $w^{i}_{a,b} = \sqrt{n} h_{i,a}h_{a,b}$.
Let us first compute the Frobenius norm of the matrix $\textbf{A}^{i,j}$, defined based on the aforementioned sequence $\{\textbf{W}^{i}\}_{i=1,...,n}$.
We have:
\begin{align}
\|\textbf{A}^{i,j}\|_{F}^{2} & = \sum_{l,t \in \{1,...,n\}} 
(\sum_{k=1}^{n} w^{j}_{k,l}w^{i}_{k,t})^{2} \notag \\
& = n^{2} \sum_{l,t \in \{1,...,n\}} (\sum_{k=1}^{n} h_{j,k}h_{k,l}h_{i,k}h_{k,t})^{2}
\end{align}
To compute the expression above, note first that for $r_{1} \neq r_{2}$ we have:
\begin{align}
\theta & = \sum_{k,l} h_{r_{1},k}h_{r_{1},l}h_{r_{2},k}h_{r_{2},l} \notag \\
 & = \sum_{k} h_{r_{1},k}h_{r_{2},k} \sum_{l} h_{r_{1},l}h_{r_{2}, l} = 0,
\end{align}
where the last equality comes from fact that different rows of $\ {H}$ are orthogonal. From the fact that $\theta = 0$ we get:
\begin{align}
\|\textbf{A}^{i,j}\|_{F}^{2} & = n^{2} \sum_{r=1,...,n} \sum_{k,l} 
h_{i,r}^{2}h_{j,r}^{2}h_{r,k}^{2}h_{r,l}^{2} \notag \\
& = n \cdot n^{2} (\frac{1}{\sqrt{n}})^{8} \cdot n^{2} = n.
\end{align}
Thus we have: $\Lambda_{F} \leq \sqrt{n}$.

Now we compute $\|\textbf{A}^{i,j}\|_{2}$.
Notice that from the definition of $\textbf{A}^{i,j}$ we get that
\begin{equation}
\textbf{A}^{i,j} = \textbf{E}^{i,j} \textbf{F}^{i,j},
\end{equation}
where the $l^{th}$ row of $\textbf{E}^{i,j}$ is of the form
$(h_{j,1}h_{1,l},...,h_{j,n}h_{n,l})$ and the $t^{th}$ column of
$\textbf{F}^{i,j}$ is of the form $(h_{i,1}h_{1,t},...,h_{i,n}h_{n,t})^{T}$.
Thus one can easily verify that $\textbf{E}^{i,j}$ and $\textbf{H}^{i,j}$ are isometries (since $\textbf{H}$ is) thus
$\textbf{A}^{i,j}$ is also an isometry and therefore $\Lambda_{2} = 1$. As in the previous setting, remaining conditions regarding matrices $\textbf{W}^{i}$ are trivially satisfied (from the basic properties of Hadamard matrices).
That completes the proof. \end{proof}

\subsubsection{Proof of Theorem \ref{main_struct_theorem}}

Let us briefly give an overview of the proof before presenting it in detail. Challenges regarding proving accuracy results for structured matrices come from the fact that, for any given $\textbf{x} \in \mathbb{R}^{n}$,
different dimensions of $\textbf{y} = \textbf{G}_{struct}\textbf{x}$ are no longer independent (as it is the case for the unstructured setting).
For matrices from the family of structured spinners we can, however, show that with high probability different elements of $\textbf{y}$ correspond to projections of a given vector $\textbf{r}$ (see Section \ref{sec:model}) into directions that are close to orthogonal. The ``close-to-orthogonality'' characteristic is obtained with the use of the Hanson-Wright inequality that focuses on concentration results regarding quadratic forms involving vectors of sub-Gaussian random variables. If $\textbf{r}$ is Gaussian, then from the well-known fact that projections of the Gaussian vector into orthogonal directions are independent, we can conclude that dimensions of $\textbf{y}$ are ``close to independent''. If $\textbf{r}$ is a discrete vector then we need to show that for $n$ large enough, it ``resembles'' the Gaussian vector. This is where we need to apply the aforementioned techniques regarding multivariate Berry-Esseen-type central limit theorem results.

\begin{proof}
We will use notation from Section \ref{sec:model} and previous sections of the Supplement. 
We assume that the model with structured matrices stacked vertically, each of $m$ rows, is applied. Without loss of generality, we can assume that we have just one block since different blocks are chosen independently.
Let $\textbf{G}_{struct}$ be a matrix from the family of structured spinners. Let us assume that $\textbf{G}_{struct}$ is used by a function $f$ operating in the $d$-dimensional space and let us denote by $\textbf{x}^{1}, \hdots, \textbf{x}^{d}$ some fixed orthonormal basis of that space.
Our first goal is to compute: $\textbf{y}^{1} = \textbf{G}_{struct} \textbf{x}^{1},...,\textbf{y}^{d} = \textbf{G}_{struct} \textbf{x}^{d}$.
Denote by $\tilde{\textbf{x}}^{i}$ the linearly transformed version of $\textbf{x}$ after applying block $\textbf{B}_{1}$, i.e. $\tilde{\textbf{x}}^{i} = \textbf{B}_{1} \textbf{x}^{i}$.
Since $\textbf{B}_{1}$ is $(\delta(n),p(n))$-balanced), we conclude that with probability at least: $p_{balanced} \geq 1 - dp(n)$ each element of each $\tilde{\textbf{x}}^{i}$ has absolute value at most $\frac{\delta(n)}{\sqrt{n}}$. We shortly say that each $\tilde{\textbf{x}}^{i}$ is $\delta(n)$-balanced.
We call this event $\mathcal{E}_{balanced}$.

Note that by the definition of structured spinners, each $\textbf{y}^{i}$ is of the form: 
\begin{equation}
\textbf{y}^{i} = (\textbf{r}^{T} \cdot \phi_{1}(\textbf{B}_{2}\tilde{\textbf{x}}^{i}),...,\textbf{r}^{T} \cdot \phi_{m}(\textbf{B}_{2}\tilde{\textbf{x}}^{i}))^{T}.
\end{equation}

For clarity and to reduce notation, we will assume that $\textbf{r}$ is $n$-dimensional.
To obtain results for vectors $\textbf{r}$ of different dimensionality $D$, it suffices to replace in our analysis and theoretical statements $n$ by $D$.
Let us denote $\mathcal{A} = \{\phi_{1}(\textbf{B}_{2}\tilde{\textbf{x}}^{1}),...,\phi_{m}(\textbf{B}_{2}\tilde{\textbf{x}}^{1}),...,\phi_{1}(\textbf{B}_{2}\tilde{\textbf{x}}^{d}),...,\phi_{m}(\textbf{B}_{2}\tilde{\textbf{x}}^{d}))\}$.
Our goal is to show that with high probability (in respect to random choices of $\textbf{B}_{1}$ and $\textbf{B}_{2}$) for all $\textbf{v}^{i},\textbf{v}^{j} \in \mathcal{A}$, $i \neq j$ the following is true:
\begin{equation}
\label{dot_product_equation}
|(\textbf{v}^{i})^{T} \cdot \textbf{v}^{j}| \leq t
\end{equation}
for some given $0 < t \ll 1$.

Fix some $t>0$. We would like to compute
the lower bound on the corresponding probability.
Let us fix two vectors $\textbf{v}^{1}, \textbf{v}^{2} \in \mathcal{A}$ and denote them as: $\textbf{v}^{1} = \phi_{i}(\textbf{B}_{2}\textbf{x})$, $\textbf{v}^{2} = \phi_{j}(\textbf{B}_{2} \textbf{y})$ for some $\textbf{x} = (x_{1},...,x_{n})^{T}$
and $\textbf{y} = (y_{1},...,y_{n})^{T}$.
Note that we have (see denotation from Section \ref{sec:model}):
\begin{multline}
\phi_{i}(\textbf{B}_{2}\textbf{x}) = (w^{i}_{11}\rho_{1}x_{1} + ...  \\
+ w^{i}_{1,n}\rho_{n}x_{n},...,w^{i}_{n,1}\rho_{1}x_{1} + ... + w^{i}_{n,n}\rho_{n}x_{n})^{T}
\end{multline}
and
\begin{multline}
\phi_{j}(\textbf{B}_{2}\textbf{y}) = (w^{j}_{11}\rho_{1}y_{1} + ... + w^{j}_{1,n}\rho_{n}y_{n},..., \\
w^{j}_{n,1}\rho_{1}y_{1} + ... + w^{j}_{n,n}\rho_{n}y_{n})^{T}.
\end{multline}
We obtain:
\begin{equation}
(\textbf{v}^{1})^{T} \cdot \textbf{v}^{2}=
\sum_{l \in \{1,...,n\}, u\in \{1,...,n\}} \rho_{l}\rho_{u}(\sum_{k=1}^{n}x_{l}y_{u}w^{i}_{k,u}w^{j}_{k,l}).
\end{equation}

We now show that, under assumptions from Theorem \ref{main_struct_theorem}, the expected 
value of the expression above is $0$.
We have:
\begin{equation}
\mathbb{E}[(\textbf{v}^{1})^{T} \cdot \textbf{v}^{2}]=
\mathbb{E}[\sum_{l \in \{1,...,n\}} \rho_{l}^{2}x_{l}y_{l}
(\sum_{k=1}^{n}w^{i}_{k,l}w^{j}_{k,l})],
\end{equation}
since $\rho_{1},...,\rho_{n}$ are independent
and have expectations equal to $0$.
Now notice that if $i \neq j$ then from the assumption that
corresponding columns of matrices $\textbf{W}^{i}$ and $\textbf{W}^{j}$ are orthogonal, we get that the above expectation is $0$. Now assume that $i = j$. But then $\textbf{x}$ and $\textbf{y}$ have to be different and thus they are orthogonal (since they are taken from the orthonormal system transformed by an isometry).
In that setting we get:
\begin{align}
\mathbb{E}[(\textbf{v}^{1})^{T} \cdot \textbf{v}^{2}] & =
\mathbb{E}[\sum_{l \in \{1,...,n\}} \rho_{l}^{2}x_{l}y_{l}
(\sum_{k=1}^{n}(w^{i}_{k,l})^{2})] \notag \\
& = \tau w \sum_{l=1}^{n} x_{l}y_{l} = 0,
\end{align}
where $\tau$ stands for the second moment of each $\rho_{i}$,
$w$ is the squared $L_{2}$-norm of each column of $\textbf{W}^{i}$
($\tau$ and $w$ are well defined due to the properties of structured spinners). The last inequality comes from the fact that $\textbf{x}$ and $\textbf{y}$ are orthogonal.
Now if we define matrices $\textbf{A}^{i,j}$ as in the definition of the model of structured spinners then we see that 
\begin{equation}
(\textbf{v}^{1})^{T} \cdot \textbf{v}^{2} = 
\sum_{l,u \in \{1,...,n\}} \rho_{l}\rho_{u}T^{i,j}_{l,u},
\end{equation}
where:
$T^{i,j}_{l,u} = x_{l}y_{u}A^{i,j}_{l,u}$.

Now we will use the following inequality:
\begin{theorem}[Hanson-Wright Inequality]
Let $\textbf{X} = (X_{1},...,X_{n})^{T} \in \mathbb{R}^{n}$ be a random vector with independent components $X_{i}$ which satisfy: $\mathbb{E}[X_{i}] = 0$ and have sub-Gaussian norm at most $K$ for some given $K>0$.
Let $\textbf{A}$ be an $n \times n$ matrix. Then for every $t \geq 0$
the following is true:
\begin{multline}
\mathbb{P}[\textbf{X}^{T}\textbf{A}\textbf{X} - \mathbb{E}[\textbf{X}^{T}\textbf{AX}] > t] \\
\leq 
2e^{-c \min(\frac{t^{2}}{K^{4}\|\textbf{A}\|^{2}_{F}},
\frac{t}{K^{2}\|\textbf{A}\|_{2}})},
\end{multline}
where $c$ is some universal positive constant.
\end{theorem}

Note that, assuming $\delta(n)$-balancedness, we have: $\|\textbf{T}^{i,j}\|_{F} \leq \frac{\delta^{2}(n)}{n} \|\textbf{A}^{i,j}\|_{F}$ and $\|\textbf{T}^{i,j}\|_{2} \leq
\frac{\delta^{2}(n)}{n}
\|\textbf{A}^{i,j}\|_{2}$.

Now we take $\textbf{X} = (\rho_{1},...,\rho_{n})^{T}$ and $\textbf{A} = \textbf{T}^{i,j}$ in the theorem above.
Applying the Hanson-Wright inequality in that setting,
taking the union bound over all pairs of different vectors $\textbf{v}^{i},\textbf{v}^{j} \in \mathcal{A}$ (this number is exactly: ${md \choose 2}$) and the event $\mathcal{E}_{balanced}$, finally taking the union bound over all $s$ functions $f_{i}$, we conclude that with probability at least: 
\begin{multline}
\label{imp_equation}
p_{good} = 1 - p(n)ds \\
- 2{md \choose 2}s e^{-\Omega(\min(\frac{t^{2}n^{2}}{K^{4}\Lambda_{F}^{2}\delta^{4}(n)}, \frac{tn}{K^{2}\Lambda_{2} \delta^{2}(n)}))}
\end{multline}
for every $f$ any two different vectors $\textbf{v}^{i}, \textbf{v}^{j} \in \mathcal{A}$ satisfy:
$|(\textbf{v}^{i})^{T} \cdot \textbf{v}^{j}| \leq t$.

Note that from the fact that $\textbf{B}_{2}\textbf{B}_{1}$ is $(\delta(n),p(n))$-balanced and from Equation \ref{imp_equation}, we get that with probability at least:
\begin{multline}
p_{right} = 1 - 2p(n)ds \\
- 2{md \choose 2}s e^{-\Omega(\min(\frac{t^{2}n^{2}}{K^{4}\Lambda_{F}^{2}\delta^{4}(n)}, \frac{tn}{K^{2}\Lambda_{2} \delta^{2}(n)}))}.
\end{multline}
for every $f$ any two different vectors $\textbf{v}^{i}, \textbf{v}^{j} \in \mathcal{A}$ satisfy:
$|(\textbf{v}^{i})^{T} \cdot \textbf{v}^{j}| \leq t$ and furthermore each $\textbf{v}^{i}$ is $\delta(n)$-balanced.

Assume now that this event happens.
Consider the vector 
\begin{equation}
\textbf{q}^{\prime} = ((\textbf{y}^{1})^{T},...,(\textbf{y}^{d})^{T})^{T} \in \mathbb{R}^{md}. 
\end{equation}
Note that $\textbf{q}^{\prime}$ can be equivalently represented as:
\begin{equation}
\textbf{q}^{\prime} = 
(\textbf{r}^{T} \cdot \textbf{v}^{1},...,\textbf{r}^{T} \cdot \textbf{v}^{md}),
\end{equation}
where: $\mathcal{A} = \{\textbf{v}^{1},...,\textbf{v}^{md}\}$.
From the fact that $\phi_{i}\textbf{B}_{2}$ and $\textbf{B}_{1}$ are isometries we conclude that: $\|\textbf{v}^{i}\|_{2} = 1$ for $i=1,...$. 

Now we will need the following Berry-Esseen type result for random vectors:

\begin{theorem}[Bentkus \cite{bentkus2003dependence}]
\label{clt_theorem}
Let $\textbf{X}_{1},...,\textbf{X}_{n}$ be independent vectors taken from $\mathbb{R}^{k}$ with common mean $\mathbb{E}[\textbf{X}_{i}] = 0$. Let $\textbf{S} = \textbf{X}_{1} + ... + \textbf{X}_{n}$.
Assume that the covariance operator $\textbf{C}^{2} = cov(\textbf{S})$ is invertible. Denote $\beta_{i} = 
\mathbb{E}[\|\textbf{C}^{-1}\textbf{X}_{i}\|_{2}^{3}]$
and $\beta = \beta_{1} + ... + \beta_{n}$.
Let $\mathcal{C}$ be the set of all convex subsets of $\mathbb{R}^{k}$. Denote $\Delta(\mathcal{C}) = \sup_{A \in \mathcal{C}} |\mathbb{P}[S \in A]-\mathbb{P}[Z \in A]|$,
where $Z$ is the multivariate Gaussian distribution with mean $0$ and covariance operator $\textbf{C}^{2}$. Then:
\begin{equation}
\Delta(\mathcal{C}) \leq ck^{\frac{1}{4}} \beta
\end{equation}
for some universal constant $c$.
\end{theorem}

Denote: $\textbf{X}_{i} = (r_{i}v^{1}_{i},...,r_{i}v^{k}_{i})^{T}$
for $k=md$, $\textbf{r} = (r_{1},...,r_{n})^{T}$
and $\textbf{v}^{j} = (v^{j}_{1},...,v^{j}_{n})$.
Note that $\textbf{q}^{\prime} = \textbf{X}_{1} + ... + \textbf{X}_{n}$. Clearly we have: $\mathbb{E}[\textbf{X}_{i}] = 0$ (the expectation is taken with respect to the random choice of \textbf{r}).
Furthermore, given the choices of $\textbf{v}^{1},...,\textbf{v}^{k}$, random vectors $\textbf{X}_{1},..,\textbf{X}_{n}$ are independent.

Let us calculate now the covariance matrix of $\textbf{q}^{\prime}$.
We have: 
\begin{equation}
\textbf{q}^{\prime}_{i} = r_{1}v^{i}_{1} + ... + r_{n}v^{i}_{n},
\end{equation}
where: $\textbf{q}^{\prime} = (\textbf{q}^{\prime}_{1},...,\textbf{q}^{\prime}_{k})$.

Thus for $i_{1}, i_{2}$ we have:
\begin{align}
\mathbb{E}[\textbf{q}^{\prime}_{i_{1}} \textbf{q}^{\prime}_{i_{2}}] & = 
\sum_{j=1}^{n} v^{i_{1}}_{j}v^{i_{2}}_{j}\mathbb{E}[r_{j}^{2}] + 2\sum_{1 \leq j_{1} < j_{2} \leq n} v^{i_{1}}_{j_{1}}v^{i_{2}}_{j_{2}} \mathbb{E}[r_{j_{1}}r_{j_{2}}] \notag \\
& = (\textbf{v}^{i_{1}})^{T} \cdot \textbf{v}^{i_{2}},
\end{align}
where the last equation comes from the fact $r_{j}$ are either Gaussian from $\mathcal{N}(0,1)$ or discrete with entries from $\{-1,+1\}$ and furthermore different $r_{j}$s are independent.

Therefore if $i_1=i_2=i$, since each $\textbf{v}^{i}$ has unit $L_{2}$-norm, we have that 
\begin{equation}
\mathbb{E}[\textbf{q}^{\prime}_{i} \textbf{q}^{\prime}_{i}] = 1, 
\end{equation}
and for $i_{1} \neq i_{2}$ we get:
\begin{equation}
|\mathbb{E}[\textbf{q}^{\prime}_{i_{1}} \textbf{q}^{\prime}_{i_{2}}]| \leq t.
\end{equation}
We conclude that the covariance matrix $\Sigma_{\textbf{q}^{\prime}}$ of the distribution $\textbf{q}^{\prime}$ is a matrix with entries $1$ on the diagonal and other entries of absolute value at most $t$.

For $t = o_{k}(1)$ small enough 
and from the $\delta(n)$-balancedness of vectors $\textbf{v}^{1},...,\textbf{v}^{k}$ 
we can conclude that:
\begin{equation}
\mathbb{E}[\|\textbf{C}^{-1}\textbf{X}_{i}\|^{3}_{2}] =
O(\mathbb{E}[\|\textbf{X}_{i}\|^{3}_{2}]) = O(\sqrt{(\frac{k}{n})^{3}}\delta^{3}(n)),
\end{equation}

Now, using Theorem \ref{clt_theorem}, we conclude that
\begin{align}
\sup_{A \in \mathcal{C}} |\mathbb{P}[\textbf{q}^{\prime} \in A] - \mathbb{P}[Z \in A]|& = O(k^{\frac{1}{4}}n \cdot \frac{k^{\frac{3}{2}}}{n^{\frac{3}{2}}} \delta^{3}(n)) \notag \\
& = O(\frac{\delta^{3}(n)}{\sqrt{n}}k^{\frac{7}{4}}),
\end{align}
where $Z$ is taken from the multivariate Gaussian distribution with covariance matrix $\textbf{I} + \textbf{E}$ and $\mathcal{C}$ is the set of all convex sets. 
Now if we apply the above inequality to the pairwise disjoint convex sets $A_{1},...,A_{j}$, where $A_{1} \cup ... \cup A_{j} = f_{i}^{-1}(\mathcal{S})$ and $l \leq b$
(such sets exist form the $b$-convexity of $f^{-1}_{i}(\mathcal{S})$),
take $\eta = \frac{\delta^{3}(n)}{\sqrt{n}}k^{\frac{7}{4}}$, $\epsilon = t = o_{md}(1)$ and take $n$ large enough, the statement of the theorem follows.
\end{proof}

\subsubsection{Proof of Theorem \ref{convex_theorem}}

\begin{proof}
Let us assume that $f_{i}$ is a convex function of $\textbf{q}_{f_{i}}$ (if $f_{i}$ is concave then the proof completely analogous).
For any $t \in \mathbb{R}$ let $\mathcal{S}_{t} = \{\textbf{q}_{f_{i}} : f_{i}(\textbf{q}_{f_{i}}) \leq t\}$ for $f_{i}$
and $\mathcal{S}_{t} = \{\textbf{q}_{f^{\prime}_{i}} : f^{\prime}_{i}(\textbf{q}_{f^{\prime}_{i}}) \leq t\}$ for $f^{\prime}_{i}$. From the convexity assumption we get that $\mathcal{S}_{t}$ is a convex set.
Thus we can directly apply Theorem \ref{main_struct_theorem} and the result
regarding cdf functions follows.
To obtain the result regarding the characteristic functions, 
notice first that we have:
\begin{equation}
\phi_{X}(t) = \int_{-1}^{1}\mathbb{P}[cos(tX) > s]ds + 
i\int_{-1}^{1}\mathbb{P}[sin(tX) > s]ds
\end{equation}
The event $\{cos(tX) > s\}$ for $t \neq 0$ is equivalent to:
$X \in \cup_{I \in \mathcal{I}} I$ for some family of intervals $\mathcal{I}$. 
Similar observation is true for the event $\{sin(tX) > s\}$.

In our scenario, from the fact that $f_{i}$ is bounded, we conclude that the corresponding families $\mathcal{I}$ are finite. Furthermore, the probability of belonging to a particular interval can be expressed by the values of the cdf function in the endpoints of that interval. From this observation and the result on cdfs that we have just obtained, the result for the characteristic functions follows immediately. 
\end{proof}

\subsubsection{Proof of Theorem \ref{corollary_theorem}}

\begin{proof} This comes directly from Theorem \ref{main_struct_theorem} and Lemma \ref{simple_lemma}.
\end{proof}

\subsubsection{Proof of Theorem \ref{hopefully_last_theorem}}

\begin{proof}
For clarity we will assume that the structured matrix consists of just one block of $m$ rows and will compare its performance with the unstructured variant of $m$ rows (the more general case when the structured matrix is obtained by stacking vertically many blocks is analogous since the blocks are chosen independently). 

Consider the two-dimensional linear space $\mathcal{H}$ spanned by $\textbf{x}$ and $\textbf{y}$.
Fix some orthonormal basis $\mathcal{B} = \{\textbf{u}^{1},\textbf{u}^{2}\}$
of $\mathcal{H}$.
Take vectors \textbf{q} and $\textbf{q}^{\prime}$. 
Note that they are $2m$-dimensional, where $m$ is the number of rows of the block used in the structured setting.
From Theorem \ref{corollary_theorem} we conclude that will probability
at least $p_{success}$, where $p_{success}$ is as in the statement of the theorem the following holds for any convex $2m$-dimensional set $A$:
\begin{equation}
|\mathbb{P}[\textbf{q}(\epsilon) \in A] - \mathbb{P}[\textbf{q}^{\prime} \in A]| \leq \eta,
\end{equation}
where $\eta = \frac{\log^{3}(n)}{n^{\frac{2}{5}}}$.
Take two corresponding entries of vectors $\textbf{v}_{\textbf{x},\textbf{y}}^{1}$ and $\textbf{v}_{\textbf{x},\textbf{y}}^{2}$ indexed by a pair $(\textbf{e}_{i}, \textbf{e}_{j})$ for some fixed $i,j \in \{1,...,m\}$ (for the case when the pair is not of the form $(\textbf{e},\textbf{e}_{j})$, but of a general form: $(\pm \textbf{e}_{i},\pm \textbf{e}_{j})$ the analysis is exactly the same).
Call them $p^{1}$ and $p^{2}$ respectively. Our goal is to compute 
$|p^{1} - p^{2}|$.
Notice that $p^{1}$ is the probability that $h(\textbf{x}) = \textbf{e}_{i}$
and $h(\textbf{y}) = \textbf{e}_{j}$ for the unstructured setting and $p^{2}$ is that probability for the structured variant.

Let us consider now the event $E^{1} = \{h(\textbf{x}) = \textbf{e}_{i}
\land h(\textbf{y}) = \textbf{e}_{j}\}$, where the setting is unstructured.
Denote the corresponding event for the structured setting as $E^{2}$.
Denote $\textbf{q} = (q_{1},...,q_{2m})$.
Assume that $\textbf{x} = \alpha_{1} \textbf{u}^{1} + \alpha_{2} \textbf{u}^{2}$ for some scalars $\alpha_{1}, \alpha_{2} > 0$.
Denote the unstructured Gaussian matrix by $\textbf{G}$.
We have:
\begin{equation}
\textbf{Gx} = \alpha_{1}\textbf{G}\textbf{u}^{1} + 
              \alpha_{2}\textbf{G}\textbf{u}^{2}
\end{equation}
Note that we have: $\textbf{G}\textbf{u}^{1} = (q_{1},...,q_{m})^{T}$
and $\textbf{G}\textbf{u}^{2} = (q_{m+1},...,q_{2m})^{T}$.
Denote by $A(\textbf{e}_{i})$ the set of all the points in $\mathbb{R}^{m}$ such that their angular distance to $\textbf{e}_{i}$ is at most the angular distance to all other $m-1$ canonical vectors. Note that this is definitely the convex set.
Now denote:
\begin{multline}
Q(\textbf{e}_{i}) = \{(q_{1},...,q_{2m})^{T} \in \mathbb{R}^{2m} : \\ \alpha_{1} (q_{1},...,q_{m})^{T} + \alpha_{2} (q_{m+1},...,q_{2m})^{T} \in A(\textbf{e}_{i})\}.
\end{multline}
Note that since $A(\textbf{e}_{i})$ is convex, we can conclude that $Q(\textbf{e}_{i})$ is also convex. 
Note that 
\begin{equation}
\{h(\textbf{x}) = \textbf{e}_{i}\} = \{\textbf{q} \in Q(\textbf{e}_{i})\}.
\end{equation}
By repeating the analysis for the event $\{h(\textbf{y}) = \textbf{e}_{j}\}$,
we conclude that:
\begin{equation}
\{h(\textbf{x}) = \textbf{e}_{i} \land h(\textbf{y}) = \textbf{e}_{j} \} = 
\{\textbf{q} \in Y(\textbf{e}_{i},\textbf{e}_{j})\}
\end{equation}
for convex set $Y(\textbf{e}_{i},\textbf{e}_{j}) = Q(\textbf{e}_{i}) \cap Q(\textbf{e}_{j})$.
Now observe that
\begin{equation}
|p^{1}-p^{2}| = 
|\mathbb{P}[\textbf{q} \in Y(\textbf{e}_{i},\textbf{e}_{j})] - 
\mathbb{P}[\textbf{q}^{\prime} \in Y(\textbf{e}_{i},\textbf{e}_{j})]|
\end{equation}
Thus we have:
\begin{multline}
|p^{1} - p^{2}| \leq 
|\mathbb{P}[\textbf{q} \in Y(\textbf{e}_{i},\textbf{e}_{j})] - 
\mathbb{P}[\textbf{q}(\epsilon) \in Y(\textbf{e}_{i},\textbf{e}_{j})]| \\ +
|\mathbb{P}[\textbf{q}(\epsilon) \in Y(\textbf{e}_{i},\textbf{e}_{j})] - 
\mathbb{P}[\textbf{q}^{\prime} \in Y(\textbf{e}_{i},\textbf{e}_{j})]|
\end{multline}
Therefore we have:
\begin{equation}
|p^{1} - p^{2}| \leq |\mathbb{P}[\textbf{q} \in Y(\textbf{e}_{i},\textbf{e}_{j})] - 
\mathbb{P}[\textbf{q}(\epsilon) \in Y(\textbf{e}_{i},\textbf{e}_{j})]| + \eta.
\end{equation}
Thus we just need to upper-bound:
\begin{equation}
\xi = |\mathbb{P}[\textbf{q} \in Y(\textbf{e}_{i},\textbf{e}_{j})] - 
\mathbb{P}[\textbf{q}(\epsilon) \in Y(\textbf{e}_{i},\textbf{e}_{j})]|.
\end{equation}
Denote the covariance matrix of the distribution $\textbf{q}(\epsilon)$
as $\textbf{I} + \textbf{E}$. Note that $\textbf{E}$ is equal to $0$ on the diagonal and the 
absolute value of all other off-diagonal entries is at most $\epsilon$.

Denote $k = 2m$. We have
\begin{align*}
& \xi = \left|A - B \right|, \\
& \mbox{where} \ A = \frac{1}{(2 \pi)^{\frac{k}{2}}\sqrt{\det (I+E)}} \int_{Y(\textbf{e}_{i},\textbf{e}_{j})} 
e^{-\frac{\textbf{x}^{T}(\textbf{I}+\textbf{E})^{-1}\textbf{x}}{2}}d\textbf{x} \\ 
& \mbox{and} \ B = \frac{1}{(2 \pi)^{\frac{k}{2}}} \int_{Y(\textbf{e}_{i},\textbf{e}_{j})} 
e^{-\frac{\textbf{x}^{T}\textbf{x}}{2}}d\textbf{x}.
\end{align*}

Expanding: $(\textbf{I}+\textbf{E})^{-1}  = \textbf{I} - \textbf{E} + \textbf{E}^{2} - ...$, noticing that $|\det(I + E) - 1| = O(\epsilon^{2m})$,
and using the above formula, we easily get:
\begin{equation}
\xi  = O(\epsilon).
\end{equation}
That completes the proof. 
\end{proof}

\subsubsection{$b$-convexity for angular kernel approximation}

Let us now consider the setting, where linear projections are used to approximate angular kernels between paris of vectors via random feature maps. In this case, the linear projection is followed by the pointwise nonlinear mapping, where the applied nonlinear mapping is a sign function.
The angular kernel is retrieved from the Hamming distance between $\{-1,+1\}$-hashes obtained in such a way.
Note that in this case we can assign to each pair $\textbf{x},\textbf{y}$ of vectors from a database a function $f_{\textbf{x},\textbf{y}}$ that outputs the binary vector which length is the size of the hash and with these indices turned on for which the hashes of $\textbf{x}$ and $\textbf{y}$ disagree. Such a  binary vector uniquely determines the Hadamard distance between the hashes. Notice that for a fixed-length hash $f_{\textbf{x},\textbf{y}}$ produces only finitely many outputs.
If $\mathcal{S}$ is a set-singleton consisting of one of the possible outputs, then
one can notice (straightforwardly from the way the hash is created) that $f^{-1}_{\textbf{x},\textbf{y}}(\mathcal{S})$ is an intersection of the convex sets (as a function of $\textbf{q}_{f_{\textbf{x},\textbf{y}}}$). Thus it is convex and thus for sets $\mathcal{S}$ which are singletons we can take $b=1$.

\subsubsection {Proof of Theorem \ref{neural_theorem}}

In this section, we show that by learning vector $\textbf{r} \in \mathbb{R}^{k}$ from the definition above, one can
approximate well any matrix $\textbf{M} \in \mathbb{R}^{m \times n}$ learned by the neural network, providing that the size $k$ or \textbf{r} is large enough in comparison with  the number of projections and the intrinsic dimensionality $d$ of the data $\mathcal{X}$. 

Take the parametrized structured spinner matrix $\textbf{M}_{struct} \in \mathbb{R}^{m \times n}$ with a learnable vector \textbf{r}. Let $\textbf{M} \in \mathbb{R}^{m \times n}$ be a matrix learned in the unstructured setting.

Let $\mathcal{B} = \{\textbf{x}^{1},...,\textbf{x}^{d}\}$ be some orthonormal basis of the linear space, where data $\mathcal{X}$ is taken from.

\begin{proof}
Note that from the definition of the parametrized structured spinner model we can conclude that
with probability at least $p_{1} = 1 - p(n)$ with respect to the choices of $\textbf{M}_{1}$ and $\textbf{M}_{2}$ each $\textbf{M}_{struct}\textbf{x}^{i}$ is of the form:
\begin{equation}
\textbf{M}_{struct}\textbf{x}^{i} = (\textbf{r}^{T} \cdot \textbf{z}_{1}(\textbf{q}^{i}),..., 
 \textbf{r}^{T} \cdot \textbf{z}_{m}(\textbf{q}^{i}))^{T}, 
\end{equation}
where each $\textbf{z}_{j}(\textbf{q}^{i})$ is of the form:
\begin{equation}
\textbf{z}_{j}(\textbf{q}^{i}) = (w^{j}_{1,1}\rho_{1}q^{i}_{1}+w^{j}_{1,n}\rho_{n}q^{i}_{n},...,
 w^{j}_{k,1}\rho_{1}q^{i}_{1}+w^{j}_{k,n}\rho_{n}q^{i}_{n})^{T}
\end{equation}

and $\mathcal{B}^{\prime} = \{\textbf{q}^{1},...,\textbf{q}^{d}\}$ is an orthonormal basis such that: $\|\textbf{q}^{i}\|_{\infty} \leq \frac{\delta(n)}{\sqrt{n}}$ for $i=1,...,n$.

Note that the system of equations: 
\begin{equation}
\textbf{M}^{struct} \textbf{x}^{i} = \textbf{M}\textbf{x}^{i}
\end{equation}

for $i=1,...,d$ has the solution in $\textbf{r}$ if the vectors from the set 
$\mathcal{A}=\{\textbf{z}_{j}(\textbf{q}^{i}):j=1,...,m, i=1,...d\}$ are independent.

Construct a matrix $\textbf{G} \in \mathbb{R}^{md \times k}$, where rows are vectors from $\mathcal{A}$. We want to show that $rank(\textbf{G}) = md$. It suffices to show that
$det (\textbf{G} \textbf{G}^{T}) \neq 0$. Denote $\textbf{B} = \textbf{G} \textbf{G}^{T}$.
Note that $B_{i,j} = (\textbf{v}^{i})^{T} \textbf{v}^{j}$, where $\mathcal{A} = \{\textbf{v}^{1},...,\textbf{v}^{md}\}$.
Take two vectors $\textbf{v}^{a},\textbf{v}^{b} \in \mathcal{A}$.
Note that from the definition of $\mathcal{A}$ we get:
\begin{equation}
(\textbf{v}^{a})^{T}\textbf{v}^{b} = \sum_{l \in \{1,...,n\}, u \in \{1,...,n\}}
\rho_{l}\rho_{u}x_{l}y_{u}(\sum_{s=1}^{k} w^{i}_{s,l}w^{j}_{s,u}) 
\end{equation}
for some $i,j$ and some vectors $\textbf{x}=(x_{1},...,x_{n})^{T}$, $\textbf{y} = (y_{1},...,y_{n})^{T}$. Furthermore, 
\begin{itemize}
\item $i=j$ and $\textbf{x}=\textbf{y}$ if $a=b$, 
\item $\|\textbf{x}\|_{2}=\|\textbf{y}\|_{2}=1$,
\item $\textbf{x}^{T}\textbf{y}=0$ or $\textbf{x}=\textbf{y}$ and $i \neq j$ for $a \neq b$.
\end{itemize}

We also have:

\begin{equation}
\mathbb{E}[(\textbf{v}^{a})^{T}\textbf{v}^{b}] = 
\mathbb{E}[\sum_{l \in \{1,...,n\}} \rho_{l}^{2}x_{l}y_{l}(\sum_{s=1}^{k}w^{i}_{s,l}w^{j}_{s,u})].
\end{equation}

From the previous observations and the properties of matrices $\textbf{W}^{1},...,\textbf{W}^{n}$ we conclude that the entries of the diagonal of $\textbf{B}$ are equal to $1$.
Furthermore, all other entries are $0$ on expectation. Using Hanson-Wright inequality, we conclude that for any $t>0$ we have: $|B_{i,j}| \leq t$ for all $i \neq j$ with probability at least:
$$p_{succ} = 1 - 2p(n)d - 2{md \choose 2}e^{-c \min(\frac{t^{2}n^{2}}{K^{4}\Lambda_{F}^{2}\delta^{4}(n)},
\frac{tn}{K^{2}\Lambda_{2}\delta^{2}(n)})}.$$




If this is the case, we let $\mathbf{\tilde{B}} \in \mathbb{R}^{(md) \times (md)}$ be a matrix with diagonal entries $\mathbf{\tilde{B}}_{i,i}=0$ and off-diagonal entries $\mathbf{\tilde{B}_{i,j}}= -\textbf{B}_{i,j}$. Furthermore, let $\textbf{B}^{*} \in \mathbb{R}^{(md) \times (md)}$ be a matrix with diagonal entries $\textbf{B}^*_{i,i}=0$ and off-diagonal entries $\textbf{B}^*_{i,j}=t$. 

Following a similar argument as in ~\cite{brent}, note that $\textbf{B}^*=t(\textbf{J}-\textbf{I})$ where \textbf{J} is the matrix of all ones (thus of rank $1$) and \textbf{I} is the identity matrix. Then the eigenvalues of $\textbf{B}^*$ are $t(md-1)$ with multiplicity $1$ and $t(0-1)$ with multiplicity $(md-1)$. We, thereby, are able to explicitly compute det($\textbf{I}-\textbf{B}^*$)=$(1-t(md-1))(1+t)^{md-1}$. 

If $\rho(\textbf{B}^*) \leq 1$, we can apply Theorem 1 of ~\cite{brent} by replacing \textbf{F} with $\textbf{B}^*$ and \textbf{E} with $\mathbf{\tilde{B}}$. For the convenience of the reader, we state their theorem here: Let $\textbf{F} \in \mathbb{R}^{n \times n}$ with non-negative entries and $\rho(F) \leq 1$. Let $\textbf{E} \in \mathbb{R}^{n \times n}$ with entries $\mid e_{i,j} \mid \leq f_{i,j}$, then det($\textbf{I}-\textbf{E}$) $\geq$ det($\textbf{I}-\textbf{F}$).

That is: if $\rho(\textbf{B}^*) \leq 1$,
then 
\begin{multline}
det(\textbf{I}-\textbf{B}^*)=(1-t(md-1))(1+t)^{md-1}\\ \leq  det(\textbf{I}-\mathbf{\tilde{B}}) =det(\textbf{B}).
\end{multline}
The final step is to observe that: \\
$\rho(\textbf{B}^*) \leq 1 \iff \max\{\mid t(md-1)\mid , \mid -t \mid\}=t(md-1) \leq 1 \iff t\leq \frac{1}{md-1}$.
Using this result, we, hence, see that $det(\textbf{B}) \geq (1-t(md-1))(1+t)^{md-1} \geq 0$, in particular $det(\textbf{B}) > 0$ for $t=\frac{1}{md}$. That completes the proof.
\end{proof}

\subsubsection{Additional experiments}

This experiment focuses on the Newton sketch approach~\cite{pilanci}, a generic optimization framework. It guarantees super-linear convergence with exponentially high probability for self-concordant functions, and a reduced computational complexity compared to the original second-order Newton method. The method relies on using a sketched 
version of the Hessian matrix, in place of the original one. In the subsequent experiment we show that matrices from the family of strucured spinners can be used for this purpose, thus can speed up several convex optimization problems solvers. 

\begin{figure*}[!htp]
\centering
\includegraphics[width=0.49\columnwidth]{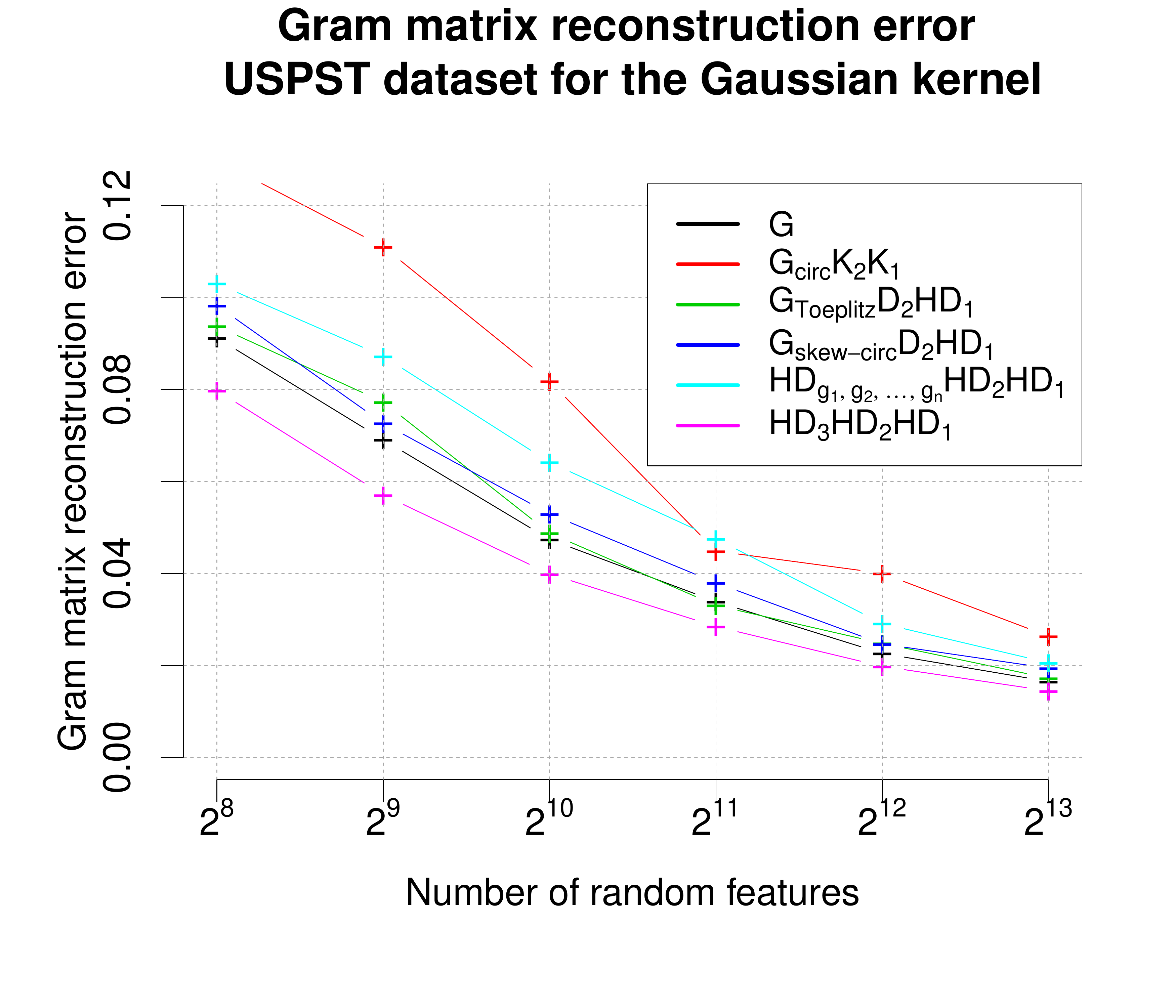} 
\includegraphics[width=0.49\columnwidth]{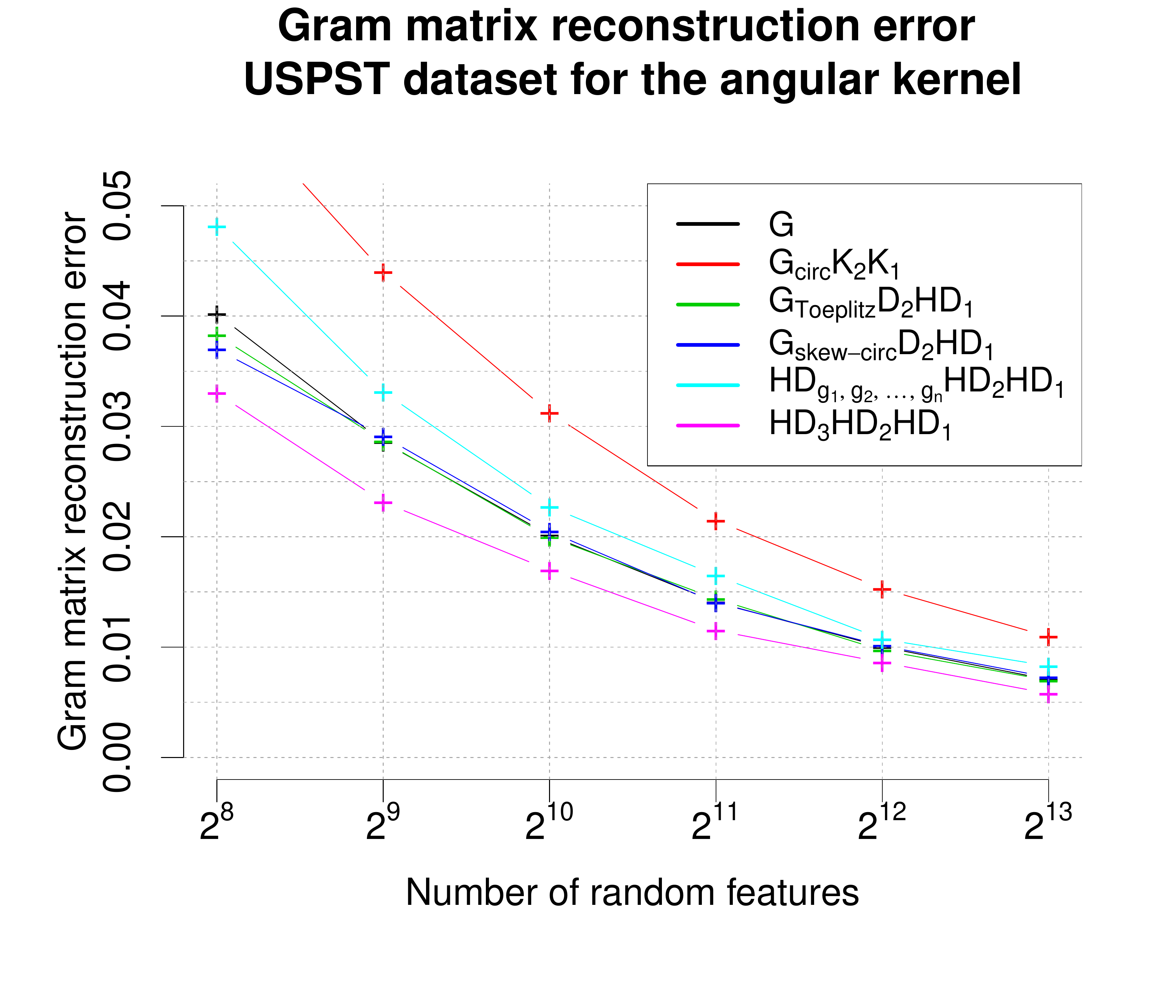}
\vspace{-0.3in}
\caption{Accuracy of random feature map kernel approximation for the USPST dataset.\vspace{-0.05in}}
\label{kernel_approx}
\end{figure*}

\begin{figure}[htp!]
\centering
\begin{subfigure}[b]{0.95\linewidth}
\centering
\includegraphics[width=\columnwidth]{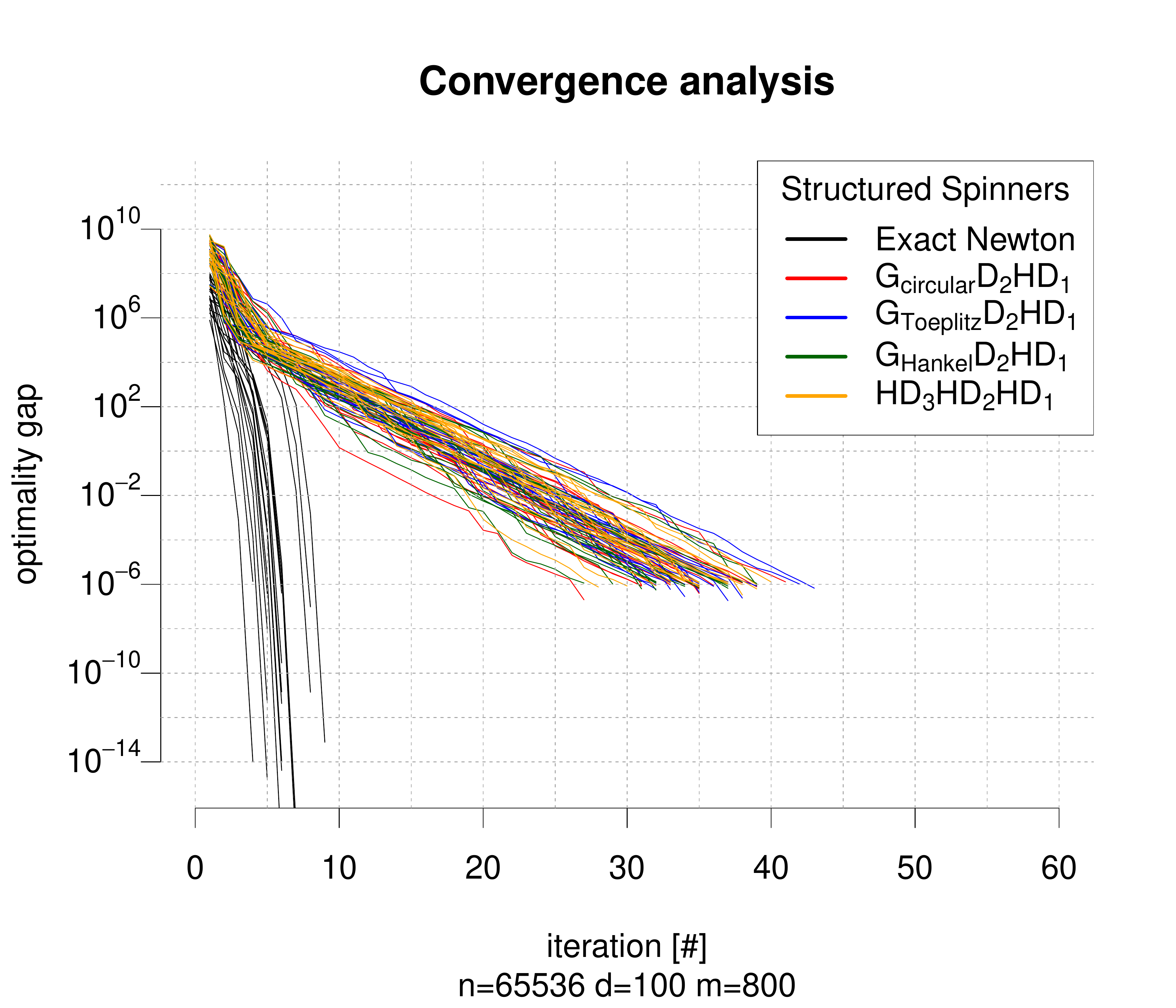}
\label{fig:newtonSketcha}
\end{subfigure}%
\\
\vspace{-0.3in}
\begin{subfigure}[b]{0.95\linewidth}
\centering
\includegraphics[width=\columnwidth]{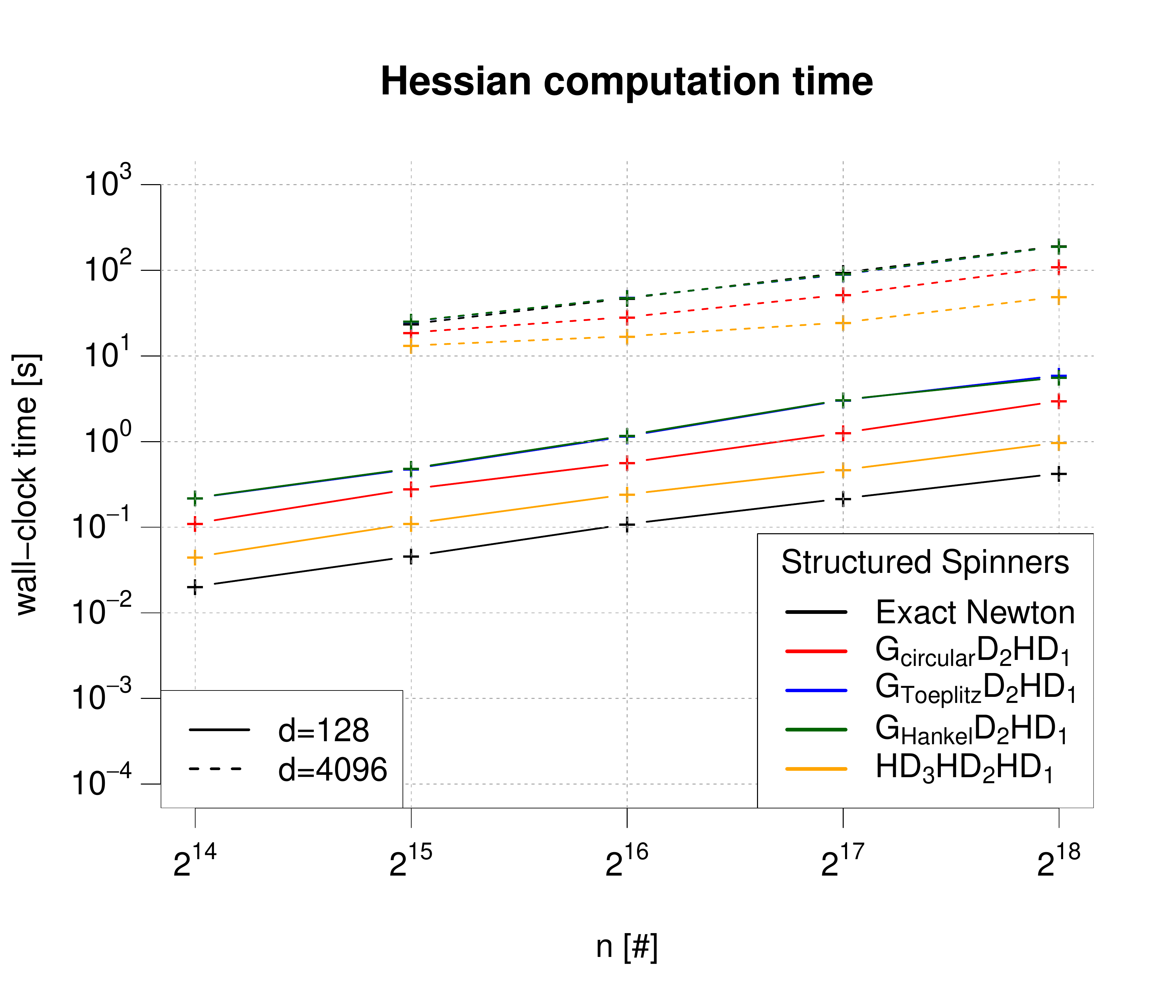}
\label{fig:newtonSketchb}
\end{subfigure}%
\vspace{-0.35in}
\caption{Numerical illustration of the convergence (top) and computational complexity (bottom) of the Newton sketch algorithm with various structured spinners. (left) Various sketching structures are compared in terms of the convergence against iteration number. (bottom) Wall-clock times of structured spinners are compared in various dimensionality settings.\vspace{-0.15in}}
\label{fig:newtonSketchConvergence}
\end{figure}

We consider the unconstrained large scale logistic regression problem, i.e. given a set of $n$ observations $\{(a_i,y_i)\}_{i=1..n}$, with $a_i \in \mathbb{R}^d$ and $y_i \in \{-1,1\}$, find $x \in \mathbb{R}^d$ minimizing the cost function 
$$ f(x) = \sum_{i=1}^n \log (1 + \exp(-y_i a_i^T x)) \enspace .$$
The Newton approach to solving this optimization problem entails solving at each iteration the least squares equation $\nabla^2 f(x^t) \Delta^t = - \nabla f(x^t) $, where

\begin{multline*}
\nabla^2 f(x^t) = \\
A^T \mathrm{diag} \left( \frac{1}{1+\exp(-a_i^T x)} (1 - \frac{1}{1+\exp(-a_i^T x)}) \right) A \\
\in \mathbb{R}^{d \times d}
\end{multline*}
 
is the Hessian matrix of $f(x^t)$, $A = [a_1^T a_2^T \cdots a_n^T] \in \mathbb{R}^{n \times d}$, $\Delta^t = x^{t+1} - x^t$ is the increment at iteration $t$ and $\nabla f(x^t) \in \mathbb{R}^d$ is the gradient of the cost function. In ~\cite{pilanci} it is proposed to consider the sketched version of the least square equation, based on a Hessian square root of $\nabla^2 f(x^t)$, denoted $\nabla^2 f(x^t)^{1/2} = \mathrm{diag} \left( \frac{1}{1+\exp(-a_i^T x)} (1 - \frac{1}{1+\exp(-a_i^T x)}) \right)^{1/2} A \in \mathbb{R}^{n \times d}$. The least squares problem at each iteration $t$ is of the form:
$$ \left( (S^t \nabla^2 f(x^t)^{1/2})^T S^t \nabla^2 f(x^t)^{1/2} \right) \Delta^t = - \nabla f(x^t) \enspace ,$$
where $S^t \in \mathbb{R}^{m \times n}$ is a sequence of isotropic sketch matrices.
Let's finally recall that the gradient of the cost function is
$$ \nabla f(x^t) = \sum_{i=1}^n \left( \frac{1}{1+\exp(-y_i a_i^T x)}-1 \right) y_i a_i \enspace . $$ 

In our experiment, the goal is to find $x \in \mathbb{R}^d$, which minimizes the logistic regression cost, given a dataset $\{(a_i,y_i)\}_{i=1..n}$, with $a_i \in 
\mathbb{R}^d$ sampled according to a Gaussian centered multivariate distribution with covariance $\Sigma_{i,j} = 0.99^{|i-j|}$ and $y_i \in \{-1,1\}$, generated at random. Various sketching matrices $S^t \in \mathbb{R}^{m \times n}$ are considered. 

In Figure~\ref{fig:newtonSketchConvergence} we report the convergence of the Newton sketch algorithm, as measured by the optimality gap defined in~\cite{pilanci}, versus the iteration number. As expected, the structured sketched versions of the algorithm do not converge as quickly as the exact Newton-sketch approach, however various matrices from the family of structured spinners exhibit equivalent convergence properties as shown in the figure.

When the dimensionality of the problem increases, the cost of computing the Hessian in the exact Newton-sketch approach becomes very large~\cite{pilanci}, scaling as $\mathcal{O}(nd^2)$. The complexity of the structured Newton-sketch approach with the matrices from the family of structured spinners is instead only $\mathcal{O}(d n \log(n) + md^2)$. Figure~\ref{fig:newtonSketchConvergence} also illustrates the wall-clock times of computing single Hessian matrices and confirms that the increase in number of iterations of the Newton sketch compared to the exact sketch is compensated by the efficiency of sketched computations, in particular Hadamard-based sketches yield improvements at the lowest dimensions.

\end{document}